\documentclass{article}

\usepackage{microtype}
\usepackage{graphicx}
\usepackage{booktabs}
\usepackage{mymacros}
\usepackage{xcolor}
\usepackage{url}
\usepackage{nicefrac}
\usepackage{algorithm2e}
\usepackage{comment}
\usepackage{enumitem}
\usepackage{tikz}
\usepackage{tabularx}
\usepackage{multirow}
\usepackage{subcaption}
\usepackage{placeins}
\usepackage{hyperref}
\usepackage[makeroom]{cancel}

\usepackage[accepted]{icml2023}

\usepackage{amsmath}
\usepackage{amssymb}
\usepackage{mathtools}
\usepackage{amsthm}
\usepackage{amsfonts}

\usepackage[capitalize,noabbrev]{cleveref}

\icmltitlerunning{Dynamical Linear Bandits}

\begin{document}

\setlength{\abovedisplayskip}{3pt}
\setlength{\belowdisplayskip}{3pt}
\setlength{\textfloatsep}{10pt}

\twocolumn[
\icmltitle{Dynamical Linear Bandits}

\icmlsetsymbol{equal}{*}

\begin{icmlauthorlist}
\icmlauthor{Marco Mussi}{poli}
\icmlauthor{Alberto Maria Metelli}{poli}
\icmlauthor{Marcello Restelli}{poli}
\end{icmlauthorlist}

\icmlaffiliation{poli}{Politecnico di Milano, Milan, Italy}
\icmlcorrespondingauthor{Marco Mussi}{marco.mussi@polimi.it}

\icmlkeywords{Online Learning, Linear Bandits}

\vskip 0.3in
]

\printAffiliationsAndNotice{}

\allowdisplaybreaks[4]

\begin{abstract}
    In many real-world sequential decision-making problems, an action does not immediately reflect on the feedback and spreads its effects over a long time frame. For instance, in online advertising, investing in a platform produces an instantaneous increase of \emph{awareness}, but the actual reward, \ie a \emph{conversion}, might occur far in the future. Furthermore, whether a conversion takes place depends on: how fast the awareness grows, its vanishing effects, and the synergy or interference with other advertising platforms. Previous work has investigated the Multi-Armed Bandit framework with the possibility of delayed and aggregated feedback, without a particular structure on how an action propagates in the future, disregarding possible dynamical effects. In this paper, we introduce a novel setting, the \settingName (\settingNameShort), an extension of the linear bandits characterized by a hidden state. When an action is performed, the learner observes a noisy reward whose mean is a linear function of the hidden state and of the action. Then, the hidden state evolves according to linear dynamics, affected by the performed action too. We start by introducing the setting, discussing the notion of optimal policy, and deriving an expected regret lower bound. Then, we provide an optimistic regret minimization algorithm, \algname (\algnameshort), that suffers an expected regret of order $\widetilde{\mathcal{O}} \Big( \frac{d \sqrt{T}}{(1-\overline{\rho})^{3/2}} \Big)$, where $\overline{\rho}$ is a measure of the stability of the system, and $d$ is the dimension of the action vector. Finally, we conduct a numerical validation on a synthetic environment and on real-world data to show the effectiveness of \algnameshort in comparison with several baselines.
\end{abstract}

\section{Introduction}\label{sec:intro}
In a large variety of sequential decision-making problems, a learner must choose an action that, when executed, determines an evolution of the underlying system state that is hidden to the learner. In these partially observable problems, the learner observes a reward (\ie feedback) representing the combined effect of multiple actions played in the past. 
For instance, in online advertising campaigns, the process that leads to a \emph{conversion}, \ie \emph{marketing funnel}~\citep{court2009consumer}, is characterized by complex dynamics and comprises several phases. When heterogeneous campaigns/platforms are involved, a profitable budget investment policy has to account for the interplay between campaigns/platforms.
In this scenario, a conversion (\eg a user's purchase of a promoted product) should be attributed not only to the latest ad the user was exposed to, but also to previous ones~\citep{berman2018beyond}. 

The \emph{joint} consideration of each funnel phase is a fundamental step towards an optimal investment solution while considering the advertising campaigns/platforms \emph{independently} leads to sub-optimal solutions. 
Consider, for instance, a simplified version of the funnel with two types of campaigns: \emph{awareness} (\ie impression) ads and \emph{conversion} ads. The first kind of ad aims at improving brand awareness, while the latter aims at creating the actual conversion. 
If we evaluate the performances in terms of conversions only, we will discover that impression ads are not instantaneously effective in creating conversions, so we will be tempted to reduce the budget invested in such a campaign. 
However, this approach is sub-optimal because impression ads increase the chance to convert when a conversion ad is shown after the impression~\citep[\eg][]{hoban2015effects}. In addition, the effect of some ads, especially impression ads delivered via television, may be delayed. 
It has been demonstrated~\citep{olivier2014modeling} that users remember advertising over time in a vanishing way, leading to consequences that non-dynamical models cannot capture. This kind of interplay comprises more general scenarios than the simple reward delay, including the case where the interaction is governed by a dynamics \emph{hidden} to the observer.

While this scenario can be indubitably modeled as a Partially Observable Markov Decision Process~\citep[POMDP,][]{aastrom1965optimal}, the complexity of the framework and its generality are often not required to capture the main features of the problem. Indeed, for specific classes of problems, the Multi-Armed Bandit~\citep[MAB,][]{lattimore2020bandit} literature has explored the possibility of experiencing delayed reward either assuming that the actual reward will be observed, individually, in the future~\citep[\eg][]{JoulaniGS13} or with the more realistic assumption that an aggregated feedback is available~\citep[\eg][]{Pike-Burke0SG18}, with also specific applications to online advertising~\citep{VernadeCP17}. Although effective in dealing with delay effects and the possibility of a reward spread in the future~\citep{Cesa-BianchiGM18}, they do not account for the additional, more complex, dynamical effects, which can be regarded as the evolution of a hidden state.

In this work, we take a different perspective. We propose to model the non-observable dynamical effects underlying the phenomena as a {Linear Time-Invariant} (LTI) system~\citep{hespanha2018linear}. In particular, the system is characterized by a hidden internal state $\xs_t$ (\eg awareness) which evolves via linear dynamics fed by the action $\us_t$ (\eg amount invested) and affected by noise. At each round, the learner experiences a reward $y_t$ (\eg conversions), which is a noisy observation that linearly combines the state $\xs_t$ and the action $\us_t$. Our goal consists in learning an optimal policy so as to maximize the expected cumulative reward. We call this setting \emph{\settingName} (DLBs) that, as we shall see, reduces to linear bandits~\cite{abbasi2011improved} when no dynamics are involved.
Because of the dynamics, the effect of each action persists over time indefinitely but, under stability conditions, it vanishes asymptotically. This allows representing interference and synergy between platforms, thanks to the dynamic nature of the system. 

\textbf{Contributions}~~
In Section~\ref{sec:problemformulation}, we introduce the Dynamical Linear Bandit (\settingNameShort) setting to represent sequential decision-making problems characterized by a hidden state that evolves linearly according to an \emph{unknown} dynamics.
We show that, under stability conditions, the optimal policy corresponds to playing the \emph{constant action} that leads the system to the most profitable steady state. Then, we derive an expected regret lower bound of order $\Omega \Big(\frac{d \sqrt{T}}{(1-\overline{\rho})^{1/2}}\Big)$, being $d$ the dimensionality of the action space and $\overline{\rho} < 1$ the spectral radius of the dynamical matrix of the system evolution law.\footnote{The smaller $\overline{\rho}$, the faster the system reaches its steady state.} In Section~\ref{sec:algorithm}, we propose a novel optimistic regret minimization algorithm, \emph{\algname} (\algnameshort), for the \settingNameShort setting. \algnameshort takes inspiration from \linucb but subdivides the optimization horizon $T$ into increasing-length epochs. In each epoch, an action is selected optimistically and kept constant (i.e., persisted) so that the system approximately reaches the steady state. 
We provide a regret analysis for \algnameshort showing that, under certain assumptions, it enjoys $\widetilde{\mathcal{O}} \Big( \frac{d \sqrt{T}}{(1-\overline{\rho})^{3/2}} \Big)$ expected regret. In Section~\ref{sec:numericalsimulations}, we provide a numerical validation, with both synthetic and real-world data, 
compared with bandit baselines. The proofs of all the results are reported in Appendix~\ref{apx:proofs_all_paper}.

\textbf{Notation}~~Let $a,b \in \Nat$ with $a \le b$, we introduce the symbols: $\llbracket a,b \rrbracket \coloneqq \{a,\dots,b\}$,  $\llbracket b \rrbracket \coloneqq \llbracket 1,b \rrbracket$, and $\llbracket a, \infty \rrparenthesis = \{a,a+1,\dots\}$ . Let $\xs,\ys \in \Reals^n$, we denote with $\inner{\xs}{\ys} = \xs^\transpose \ys = \sum_{j=1}^n x_i y_i$ the inner product. For a positive semidefinite matrix $\As \in \Reals^{n\times n}$, we denote with $\|\xs\|_{\As}^2 = \xs^\transpose \As \xs$ the weighted $2$-norm. The \emph{spectral radius} $\rho(\As)$ is the largest absolute value of the eigenvalues of $\As$, the \emph{spectral norm} $\|\As\|_2$ is the square root of the maximum eigenvalue of $\As^\transpose \As$.
We introduce the maximum spectral norm to spectral radius ratio of the powers of $\As$ defined as $\Phi(\As) = \sup_{\tau \ge 0}\|\As^\tau\|_2/{\rho(\As)^\tau}$~\citep{oymak2O19nonasymptotic}. We denote with $\Is_n$ the identity matrix of order $n$ and with $\mathbf{0}_n$ the vector of all zeros of dimension $n$. A random vector $\xs \in \Reals^n$ is $\sigma^2$-subgaussian, in the sense of~\citet{hsu2012tail}, if for every vector $\bm{\zeta} \in \Reals^{n}$ it holds that $\E\left[\exp\left(\inner{\bm{\zeta}}{\xs}\right)\right] \le \exp (\|\bm{\zeta} \|_2^2 \sigma^2/2 )$.
\section{Setting}\label{sec:problemformulation}
In this section, we introduce the \emph{\settingName} (DLBs), the learner-environment interaction, assumptions, and regret (Section~\ref{sec:setting}). Then, we derive a closed-form expression for the optimal policy for DLBs (Section~\ref{sec:optimalPolicies}). Finally, we derive a lower bound to the regret, highlighting the intrinsic complexities of the DLB setting (Section~\ref{sec:lb}).

\subsection{Problem Formulation}\label{sec:setting}
In a Dynamical Linear Bandit (DLB), the environment is characterized by a \emph{hidden} state, \ie a $n$-dimensional real vector, initialized to $\xs_1 \in \Xs $, where $\Xs\subseteq \Reals^{n}$ is the state space. At each round $t \in \Nat $, the environment is in the hidden state $\xs_t \in \Xs$, the learner chooses an action, \ie a $d$-dimensional real vector $\us_t \in \mathcal{U} $, where $\Us\subseteq \Reals^d$ is the action space. Then, the learner receives a noisy reward $y_t = \inner{\vomega}{\xs_t} + \inner{\vtheta}{\us_t} + \eta_t \in \Ys $, where $\Ys \subseteq \Reals$ is the reward space, $\vomega \in \Reals^{n}$, $\vtheta \in \Reals^{d}$ are unknown parameters, and $\eta_t$ is a zero-mean $\sigma^2$--subgaussian random noise, conditioned to the past. Then, the environment evolves to the new state according to the unknown linear dynamics $\xs_{t+1} = \As \xs_t + \Bs \us_t + \epsilons_t$, where $\As \in \Reals^{n\times n}$ is the dynamic matrix, $\Bs \in \Reals^{n \times d}$ is the action-state matrix, and $\epsilons_t $ is a zero-mean $\sigma^2$--subgaussian random noise, conditioned to the past, independent of $\eta_t$.\footnote{$n$ is the \emph{order} of the LTI system~\citep{kalman1963mathematical}. We make no assumption on the value of $n$ and on its knowledge.} 

\begin{remark}
	The setting proposed above is a particular case of a POMDP~\citep{astrom1965optimal}, in which the state $\xs_t$ is non-observable, while the learner accesses the noisy observation $y_t$ that corresponds to the noisy reward too. Furthermore, the setting can be viewed as a MISO (Multiple Input Single Output) discrete-time LTI system~\citep{kalman1963mathematical}. Finally, the DLB reduces to (non-contextual) linear bandit~\citep{abbasi2011improved} when the hidden state does not affect the reward, \ie when $\vomega = \mathbf{0}$.
\end{remark}

\textbf{Markov Parameters}~~We revise a useful representation, that for every $H \in \dsb{t}$ allows expressing $y_t$ in terms of the sequence of the most recent $H+1$ actions $(\us_s)_{s \in \dsb{t-H,t}}$, reward noise $\eta_t$, $H$ state noises $(\epsilons_s)_{s \in \dsb{t-H,t-1}}$, and starting state $\xs_{t-H}$~\citep{ho1966effective, oymak2O19nonasymptotic,tsiamis2019finite, sarkar2021finite}:
{\thinmuskip=1mu
\medmuskip=1mu
\thickmuskip=1mu
\begin{equation}\label{eq:decompY}
\resizebox{.9\linewidth}{!}{$\displaystyle
	y_t = \underbrace{\sum_{s=0}^{H} \inner{\hs^{\{s\}}}{ \us_{t-s}}}_{\text{action effect}}+ \underbrace{\vomega^\transpose \As^{H}\xs_{t-H}}_{\text{starting state}} + \underbrace{\eta_t + \sum_{s=1}^{H} \vomega^\transpose \As^{s-1}\epsilons_{t-s}}_{\text{noise}},$}
\end{equation}}%
where the sequence of vectors $ \hs^{\{s\}} \in \Reals^{ d}$ for every $s \in \Nat$ are called \emph{Markov parameters} and are defined as: $\hs^{\{0\}} = \vtheta$ and $\hs^{\{s\}} =\Bs^\transpose (\As^{s-1})^\transpose\vomega $ if $s \ge 1$. 
Furthermore, we introduce the \emph{cumulative Markov parameters}, defined for every $s,s' \in \Nat $ with $s \le s'$ as $\hs^{\dsb{s,s'}} = \sum_{l=s}^{s'} \hs^{\{l\}}$ and the corresponding limit as $s' \rightarrow +\infty$, \ie $\hs^{\llbracket s, +\infty \rrparenthesis} = \sum_{l=s}^{+\infty} \hs^{\{l\}} $. Finally, we use the abbreviation $\hs = \hs^{\llbracket 0, +\infty \rrparenthesis} =  \vtheta + \Bs^\transpose (\Is_n - \As)^{-\transpose }\vomega $. 

We will make use of the following standard assumption related to the \emph{stability} of the dynamic matrix $\As$, widely employed in discrete--time LTI literature~\citep{oymak2O19nonasymptotic,lale2020logarithmic,lale2020regret}.

\begin{ass}[Stability]\label{ass:spectralNorm}
The spectral radius of $\As$ is strictly smaller than $1$, \ie $\rho(\As) < 1$, and the maximum spectral norm to spectral radius ratio of the powers of $\As$ is bounded, \ie $\Phi(\As) < +\infty$.\footnote{The latter is a mild assumption: if $\As$ is diagonalizable as $\As = \mathbf{Q} \mathbf{\Lambda} \mathbf{Q}^{-1}$, then $\Phi(\As) \le \| \mathbf{Q}\|_2\| \mathbf{Q}^{-1}\|_2$ and it is finite. In particular, if $\As$ is symmetric then $\Phi(\As) = 1$.}
\end{ass}

\textbf{Policies and Performance}~~
The learner's behavior is modeled via a deterministic \emph{policy} $ \underline{\vpi} = (\vpi_t)_{t \in \Nat}$ defined, for every round $t \in \Nat$, as $\vpi_t : \Hs_{t-1} \rightarrow \Us$, mapping the history of observations $H_{t-1}=(\us_1,y_1,\dots,\us_{t-1},y_{t-1}) \in \Hs_{t-1}$ to an action $\us_t = \vpi_t(H_{t-1}) \in \Us$, where $\Hs_{t-1} = (\Us \times \Ys)^{t-1}$ is the set of histories of length $t-1$. 
The performance of a policy $\bm{\underline{\pi}}$ is evaluated in terms of the \emph{(infinite-horizon) expected average reward}:
\begin{align}\label{eq:eqDyn}
	&\qquad \qquad  J(\underline{\vpi}) \coloneqq  \liminf_{H \rightarrow + \infty} \E \left[\frac{1}{H} \sum_{t=1}^H y_t \right],  \\
	& \text{where} \qquad \begin{cases} \xs_{t+1} = \As \xs_t + \Bs \us_t + \epsilons_t\\
	y_t =  \inner{\vomega}{\xs_t} + \inner{\vtheta}{\us_t} + \eta_t\\
	\us_t = \vpi_t(H_{t-1})
	\end{cases}, \quad \forall t \in \Nat, \nonumber
\end{align}
where the expectation is taken \wrt the randomness of the state noise $\epsilons_t$ and reward noise $\eta_t$. If a policy $\underline{\vpi}$ is \emph{constant}, \ie $\vpi_t(H_{t-1}) = \us$ for every $t \in \Nat$, we abbreviate $J(\us) = J(\underline{\vpi})$. A policy $\underline{\vpi}^*$ is an \emph{optimal policy} if it maximizes the expected average reward, \ie $\underline{\vpi}^* \in \argmax_{\underline{\vpi} } J(\underline{\vpi})$, and its performance is denoted by $J^* \coloneqq J(\underline{\vpi}^*)$.

We further introduce the following assumption that requires the boundedness of the norms of the relevant quantities.

\begin{ass}[Boundedness]\label{ass:boundedness}
There exist $\Theta,\Omega,B,U <+\infty$ s.t.: $\|\vtheta\|_2 \le \Theta$, $\|\vomega\|_2 \le \Omega$, $\| \Bs \|_2 \le B$, $\sup_{\us \in \Us} \|\us\|_2 \le U$, and $\sup_{\xs \in \Xs} \|\xs\|_2 \le X $, $\sup_{\us \in \Us} \, |J(\us)| \le 1$.\footnote{The assumption of the bounded state norm $\|\xs\|_2\le X$ holds whenever the state noise $\epsilons$ is bounded. As shown by~\citet{AgarwalHS19}, this assumption can be relaxed, for unbounded subgaussian noise, by conditioning to the event that none of the noise vectors are ever large at the cost of an additional $\log T $ factor in the regret.}
\end{ass}

\textbf{Regret}~~The \emph{regret} suffered by playing a policy $\underline{\vpi}$, competing against the optimal infinite-horizon policy $\underline{\vpi}^*$ over a \emph{learning horizon} $T \in \Nat$ is given by: 
\begin{equation}
R(\underline{\vpi}, T) \coloneqq T J^* -  \sum_{t=1}^T  y_t,
\end{equation}
where $y_t$ is the sequence of rewards collected by playing $\underline{\vpi}$ as in Equation~\eqref{eq:eqDyn}. The goal of the learner consists in minimizing the \emph{expected regret} $\E{R}(\underline{\vpi}, T)$, where the expectation is taken \wrt the randomness of the reward.

\subsection{Optimal Policy}\label{sec:optimalPolicies}
In this section, we derive a closed-form expression for the optimal policy $\underline{\vpi}^*$ for the infinite--horizon objective function, as introduced in Equation~\eqref{eq:eqDyn}. 
\begin{restatable}[Optimal Policy]{thr}{optimalPolicy}
Under Assumptions~\ref{ass:spectralNorm} and~\ref{ass:boundedness}, an optimal policy $\underline{\vpi}^*$ maximizing the (infinite-horizon) expected average reward $J(\underline{\vpi})$ (Equation~\ref{eq:eqDyn}), for every round $ t \in \Nat$ and history $ H_{t-1} \in \Hs_{t-1}$ is given by:
{\thinmuskip=1mu
\medmuskip=1mu
\thickmuskip=1mu
\begin{align}\label{eq:markovDec}
	\vpi_{t}^*(H_{t-1}) = \us^* \;\;\;\;\;\;\;\;	\text{  where  } \;\;\;\;\;\;\;\; \us^* \in \argmax_{\us \in \Us}  J(\us) = \inner{\hs}{\us}.
\end{align}}
\end{restatable}
Some remarks are in order. The optimal policy plays the \emph{constant} action $\us^*\in \Us$ which brings the system in the \quotes{most profitable} steady-state.\footnote{In Appendix~\ref{apx:finiteHorizon}, we show that the optimal policy is non--stationary for the finite--horizon case.} Indeed, the expression $\inner{\hs}{\us}$ can be rewritten expanding the cumulative Markov parameter as $(\vtheta^\transpose + \vomega^\transpose(\Is_n - \As )^{-1}\Bs) \us^*$ and $\overline{\xs}^*=(\Is_n - \As )^{-1}\Bs \us^*$ is the expression of the steady state $\overline{\xs}^* = \As \overline{\xs}^* + \Bs \us^*$, when applying action $\us^*$. It is worth noting the role of Assumption~\ref{ass:spectralNorm} which guarantees the existence of the inverse $(\Is_n - \As )^{-1}$. In this sense, our problem shares the constant nature of the optimal policy with the linear bandit setting~\citep{abbasi2011improved}, although ours is characterized by an evolving state, which introduces a new trade-off in the action selection. From the LTI system perspective, this implies that we can restrict to \emph{open-loop stationary} policies. The reason why DLBs do not benefit from \emph{closed-loop} policies, differently from other classical problems, such as the LQG~\citep{abbasi2011regret}, lies in the linearity of the reward $y_t$ and in the additive noise $\eta_t$ and $\epsilons_t$, making their presence irrelevant (in expectation) for control purposes. 
Nonetheless, as we shall see, our problem poses additional challenges compared to linear bandits since, in order to assess the quality of an action $\us \in \Us$, instantaneous rewards are not reliable, and we need to let the system evolve to the steady state and, only then, observe the reward.

\subsection{Regret Lower Bound}\label{sec:lb}
In this section, we provide a lower bound to the expected regret that any learning algorithm suffers when addressing the learning problem in a DLB.

\begin{restatable}[Lower Bound]{thr}{lb}\label{thr:lb}
For any policy $\underline{\bm{\pi}}$ (even stochastic), there exists a DLB fulfilling Assumptions~\ref{ass:spectralNorm} and~\ref{ass:boundedness}, such that for sufficiently large $T \ge \mathcal{O} \Big( \frac{d^2}{1-\rho(\As)} \Big)$, policy $\underline{\bm{\pi}}$ suffers an expected regret lower bounded by:
\begin{align*}
	\mathbb{E} R(\underline{\bm{\pi}}, T) \ge \Omega \left( \frac{d\sqrt{T}}{(1-\rho(\As))^{\frac{1}{2}}}  \right).
\end{align*}
\end{restatable}

The lower bound highlights the main challenges of the DLB learning problem. First of all, we observe a dependence on $1/(1-\rho(\As))$, being $\rho(\As)$ the spectral radius of the matrix $\As$. This is in line with the intuition that, as $\rho(\As)$ approaches $1$, the problem becomes more challenging. Furthermore, we note that when $\rho(\As)=0$, \ie the problem has no dynamical effects, the lower bound matches the one of linear bandits~\citep{lattimore2020bandit}.
It is worth noting that, for technical reasons, the result of Theorem~\ref{thr:lb} is derived under the assumption that, at every round $t \in \dsb{T}$, the agent observes \emph{both} the state $\xs_t$ and the reward $y_t$ (see Appendix~\ref{apx:proofs_all_paper}). Clearly, this represents a simpler setting \wrt DLBs (in which $\xs_t$ is hidden) and, consequently, Theorem~\ref{thr:lb} is a viable lower bound for DLBs too.
\section{Algorithm}\label{sec:algorithm}
In this section, we present an \emph{optimistic} regret minimization algorithm for the \settingNameShort setting. \emph{\algname} (\algnameshort), whose pseudocode is reported in Algorithm~\ref{alg:alg}, requires the knowledge of an upper-bound $\overline{\rho} < 1$ on the spectral radius of the dynamic matrix $\As$ (\ie $\rho(\As) \le \overline{\rho}$) and on the  maximum spectral norm to spectral radius ratio  $\overline{\Phi} < +\infty$ (\ie $\Phi(\As) \le \overline{\Phi}$), as well as the bounds on the relevant quantities of Assumption~\ref{ass:boundedness}.\footnote{As an alternative, one can consider a more demanding requirement of the knowledge of a bound on the spectral norm $\|\As\|_2$ of $\As$. Similar assumptions regarding the knowledge of analogous quantities are considered in the literature, \eg \emph{decay of Markov operator norms}~\citep{simchowitz2020improper} and \emph{strong stability}~\citep{plevrakis2020geometric}, spectral norm bound~\citep{lale2020logarithmic}. As a side note, the knowledge of  $\overline{\rho} \ge \rho(\As)$ (or an equivalent quantity) is proved to be unavoidable by Theorem~\ref{thr:lb}. Indeed, if no restriction on $\rho(\As)$ is enforced (\ie just $\rho(\As) < 1$), one can always consider the DLB in which $\rho(\As) = 1-1/T < 1$ making the regret lower bound degenerate to linear.}
\algnameshort is based on the following simple observation. To assess the quality of action $\us \in \Us$, we need to \emph{persist} in applying it so that the system approximately reaches the corresponding steady state and, then, observe the reward $y_t$, representing a reliable estimate of $J(\us) = \inner{\hs}{\us}$. We shall show that, under Assumption~\ref{ass:spectralNorm}, the number of rounds needed to approximately reach such a steady state is logarithmic in the learning horizon $T$ and depends on the upper bound of the spectral norm $\overline{\rho}$. 
After initializing the Gram matrix $\Vs_0 = \lambda \Is_d$ and the vectors $\bs_0$ and $\widehat{\hs}_0$ both to $\mathbf{0}_d$ (line~\ref{line:init}), \algnameshort subdivides the learning horizon $T$ in $M \leq T$ \emph{epochs}. Each epoch $m \in \dsb{M}$ is composed of $H_m+1$ rounds, where $H_m =  \lfloor \log m / \log (1/\overline{\rho}) \rfloor$ is logarithmic in the epoch index $m$. At the beginning of each epoch, $m \in \dsb{M}$, \algnameshort computes the upper confidence bound (UCB) index (line~\ref{line:ucb}) defined for every $\us \in \Us$ as:
\begin{align}\label{eq:ucbEq}
	\text{UCB}_t(\us) \coloneqq \, \inner{\widehat{\hs}_{t-1}}{\us} + \beta_{t-1} \left\| \us \right\|_{\Vs_{t-1}^{-1}},
\end{align}
where $\widehat{\hs}_{t-1} = \Vs_{t-1}^{-1}\bs_{t-1}$ is the Ridge regression estimator of the cumulative Markov parameter $\hs$, as in Equation~\eqref{eq:markovDec} and $\beta_{t-1} \ge 0$ is an exploration coefficient to be defined later. Similar to \linucb~\citep{abbasi2011improved}, the index $\text{UCB}_t(\us)$ is designed to be optimistic, \ie $J(\us) \le \text{UCB}_t(\us)$ in high-probability for all $\us \in \Us$. Then, the optimistic action $\us_t \in \argmax_{\us \in \Us} \text{UCB}_t(\us)$ is executed (line~\ref{line:play}) and persisted for the next $H_m$ rounds (lines~\ref{line:for}-\ref{line:persist}). The length of the epoch $H_m$ is selected such that, under Assumption~\ref{ass:spectralNorm}, the system has approximately reached the steady state after $H_m+1$ rounds. In this way, at the end of epoch $m$, the reward $y_{t}$ is an almost-unbiased sample of the steady-state performance $J(\us_t)$. This sample is employed to update the Gram matrix estimate $\Vs_t$ and the vector $\bs_t$ (line~\ref{line:update}), while the samples collected in the previous $H_m$ rounds are discarded (line~\ref{line:discard}).
It is worth noting that by setting $H_m=0$ for all $m \in \dsb{M}$, \algnameshort reduces to \linucb. The following sections provide the concentration of the estimator $\widehat{\hs}_{t-1}$ of $\hs$ (Section~\ref{sec:selfNorm}) and the regret analysis of \algnameshort (Section~\ref{sec:regret}).

\RestyleAlgo{ruled}
\LinesNumbered
\begin{algorithm}[t]
\caption{\algnameshort.}\label{alg:alg}
\small
\SetKwInOut{Input}{Input}
\Input{Regularization parameter $\lambda > 0$, exploration coefficients $(\beta_{t-1})_{t \in \dsb{T}}$, spectral radius upper bound $ 0 \le \overline{\rho} < 1$}

Initialize $t \leftarrow 1$, $\Vs_0 = \lambda \Is_{d}$, $\bs_0 = \mathbf{0}_{d}$, $\widehat{\hs}_0 = \mathbf{0}_{d}$, \label{line:init}

{\thinmuskip=2mu
\medmuskip=2mu
\thickmuskip=2mu
Define  $M = \min\{M' \in \Nat : \sum_{m=1}^{M'} 1 + \lfloor \frac{\log m}{ \log (1/\overline{\rho})} \rfloor > T \} - 1$ \label{line:define_M}}

\For{$m \in \dsb{M}$}{
Compute $\us_{t} \in \argmax_{\us \in \Us} \text{UCB}_t(\us)$ \label{line:ucb}

{\thinmuskip=2mu
\medmuskip=2mu
\thickmuskip=2mu
$\qquad$ where $\text{UCB}_t(\us) \coloneqq \inner{\widehat{\hs}_{t-1}}{\us} + \beta_{t-1} \left\| \us \right\|_{\Vs_{t-1}^{-1}}$
}

Play arm $\us_t$ and observe $y_{t}$\label{line:play}

Define $H_m = \lfloor \frac{\log m}{\log (1/\overline{\rho})} \rfloor$

\For{$j \in \dsb{H_m}$}{ \label{line:for}

Update $\text{ } \Vs_{t} = \Vs_{t-1}, \; \bs_{t} = \bs_{t-1}$\label{line:discard}

$t \leftarrow t+1$

Play arm $\us_t = \us_{t-1}$ and observe $y_{t}$\label{line:persist}
}

Update $\text{ } \Vs_{t} = \Vs_{t-1} + \us_{t}\us_{t}^\transpose, \; \bs_{t} = \bs_{t-1} + \us_{t}y_{t}$\label{line:update}

Compute $\widehat{\hs}_t = \Vs_{t}^{-1} \bs_{t}$

$t \leftarrow t+1$

}
\end{algorithm}

\subsection{Self-Normalized Concentration Inequality for the Cumulative Markov Parameter}\label{sec:selfNorm}
In this section, we provide a self-normalized concentration result for the estimate $\widehat{\hs}_t$ of the cumulative Markov parameter $\hs$. For every epoch $m \in \dsb{M}$, we denote with $t_m$ the last round of epoch $m$: $t_0 = 0$ and  $ t_{m} = t_{m-1} + 1 + H_m$. At the end of each epoch $m$, we solve the Ridge regression problem, defined for every round $ t \in  \dsb{T}$ as:
{\thinmuskip=1mu
\medmuskip=1mu
\thickmuskip=1mu
\begin{align*}
    \widehat{\hs}_{t}  = \argmin_{\widetilde{\hs} \in \Reals^{d}} \sum_{l \in \dsb{M} : t_l \le t_m} (y_{t_l} - \inner{\widetilde{\hs}}{\us_{t_l}} )^2 + \lambda \big\| \widetilde{\hs} \big\|_2^2 = \Vs_t^{-1} \bs_t.
\end{align*} }

We now present the following self-normalized maximal concentration inequality and, then, we compare it with the existing results in the literature. 

\begin{restatable}[Self-Normalized Concentration]{thr}{concentration}\label{thr:concentration}
Let $(\widehat{\hs}_t)_{t \in \Nat}$ be the sequence of solutions of the Ridge regression problems of Algorithm~\ref{alg:alg}. Then, under Assumption~\ref{ass:spectralNorm} and~\ref{ass:boundedness}, for every $\lambda \ge 0$ and $\delta \in (0,1)$, with probability at least $1-\delta$, simultaneously for all rounds $t \in \Nat$, it holds that:
\begin{align*}
	\left\| \widehat{\hs}_t - \hs \right\|_{\Vs_{t}} & \le   \frac{c_1}{\sqrt{\lambda}} \log(e(t+1)) + c_2 \sqrt{\lambda}  \\
	& \quad + \sqrt{2  \widetilde{\sigma}^2 \left( \log \left(\frac{1}{\delta} \right) + \frac{1}{2} \log \left( \frac{\det\left({\Vs}_t\right)}{\lambda^d} \right) \right)},
\end{align*}
where $c_1 = U \Omega \Phi(\As) \left(\frac{U B }{1-\rho(\As)} + X \right)$, $c_2 = \Theta + \frac{\Omega B \Phi(\As)}{1-\rho(\As)}$, and $\widetilde{\sigma}^2 =  \sigma^2 \left( 1 + \frac{\Omega^2 \Phi(\As)^2}{1-\rho(\As)^2}\right)$.
\end{restatable}

First, we note that when $\Omega = 0$ ($\vomega = \mathbf{0}_n$), \ie the state does not affect the reward, the bound perfectly reduces to the self-normalized concentration used in linear bandits~\citep[][Theorem 1]{abbasi2011improved}. In particular, we recognize the second term due to the regularization parameter $\lambda > 0$ and the third one, which involves the subgaussianity parameter $\widetilde{\sigma}^2$, related to the joint contribution of the state and reward noises. Furthermore, the first term is an additional bias that derives from the epochs of length $H_m+1$. The choice of the value $H_m$ represents one of the main technical novelties that, on the one hand, leads to a bias that conveniently grows logarithmically with $t$ and, on the other hand, can be computed without the knowledge of $T$.

It is worth looking at our result from the perspective of learning the LTI system parameters. We can compare our Theorem~\ref{thr:concentration} with the concentration presented in~\citep[][Appendix C]{lale2020logarithmic}, which represents, to the best of our knowledge, the only result for the closed-loop identification of LTI systems with non-observable states. First, note that, although we focus on a MISO system ($y_t$ is a scalar, being our reward), extending our estimator to multiple-outputs (MIMO) is straightforward. Second, the approach of~\citep{lale2020logarithmic} employs the \emph{predictive form} of the LTI system to cope with the correlation introduced by closed-loop control. This choice allows for convenient analysis of the estimated Markov parameters of the predictive form. However, recovering the parameters of the original system requires an application of the Ho-Kalman method~\citep{ho1966effective} which, unfortunately, does not preserve the concentration properties in general, but only for \emph{persistently exciting} actions. Our method, instead, forces to play an open-loop policy within a single epoch (each with logarithmic duration), while the overall behavior is closed-loop, as the next action depends on the previous-epoch estimates. In this way, we are able to provide a concentration guarantee on the parameters of the original system without assuming additional properties on the action signal.

\subsection{Regret Analysis}\label{sec:regret}
In this section, we provide the analysis of the regret of \algnameshort, when we select the exploration coefficient $\beta_t$ based on the knowledge of the upper bounds $\overline{\rho} < 1$, $\overline{\Phi} < +\infty$, and those specified in Assumption~\ref{ass:boundedness}, defined for every round $t \in \dsb{T}$ as:
\begin{equation}\label{eq:beta}
\begin{aligned}
\beta_t & \coloneqq  \frac{\overline{c}_1}{\sqrt{\lambda}} \log(e(t+1)) + \overline{c}_2 \sqrt{\lambda} \\
	& \quad + \sqrt{2 \overline{\sigma}^2 \left( \log \left(\frac{1}{\delta} \right) + \frac{d}{2} \log \left(1 + \frac{t U^2}{d \lambda} \right) \right)}, \nonumber
\end{aligned}
\end{equation}
where $\overline{c}_1 = U \Omega \overline{\Phi} \left(\frac{U B }{1-\overline{\rho}} + X \right)$, $\overline{c}_2 =  \Theta + \frac{\Omega B \overline{\Phi}}{1-\overline{\rho}}$, and $\overline{\sigma}^2 = \sigma^2 \left( 1 + \frac{\Omega^2 \overline{\Phi}^2}{1-\overline{\rho}^2} \right)$. The following result provides the bound on the expected regret of \algnameshort. 
\begin{restatable}[Upper Bound]{thr}{regretThr}\label{thr:regretThr}
    Under Assumptions~\ref{ass:spectralNorm} and~\ref{ass:boundedness}, selecting $\beta_t$ as in Equation~\eqref{eq:beta} and $\delta=1/T$, \algnameshort suffers an expected regret bounded as (highlighting the dependencies on $T$, $\overline{\rho}$, $d$, and $\sigma$ only):
  \begin{align*}
        \E R&(\bm{\underline{\pi}}^{\text{\algnameshort}},T) \le {\mathcal{O}} \Bigg(  \frac{d \sigma \sqrt{T} (\log T)^{\frac{3}{2}}}{1-\overline{\rho}} \\
        & \quad\qquad + \frac{\sqrt{d T} (\log T)^2}{(1-\overline{\rho})^{\frac{3}{2}}} + \frac{1}{(1-\rho(\As))^2}\Bigg).
    \end{align*}
\end{restatable}

\begin{proofsketch}{\thinmuskip=2mu
\medmuskip=2mu
\thickmuskip=2mu
The analysis of \algnameshort poses additional challenges compared to that of \linucb~\cite{abbasi2011improved} because of the dynamic effects of the hidden state. 
The idea behind the proof is to first derive a bound on a different notion of regret, \ie the \emph{offline regret}: $R^{\text{off}}(\bm{\underline{\pi}},T) =  T J^* - \sum_{t=1}^T J(\us_t)$, that compares $J^*$ with the steady-state performance $J(\us_t)$ of the action $\us_t = \bm{\pi}_t(H_{t-1})$ (Theorem~\ref{thr:regretOff}). This analysis of $R^{\text{off}}(\bm{\underline{\pi}},T)$ can be comfortably carried out, by adopting a proof strategy similar to that of $\texttt{Lin-UCB}$. However, when applying action $\us_t$, the DLB does not immediately reach the performance $J(\us_t)$ as the expected reward $\E[y_t]$ experiences a transitional phase before converging to the steady state. Under stability (Assumption~\ref{ass:spectralNorm}), it is possible to show that the expected offline regret and the expected regret differ by a constant: $|\E R(\bm{\underline{\pi}},T) - \E R^{\text{off}}(\bm{\underline{\pi}},T)| \le \mathcal{O} (1/(1-\rho(\As))^2)$ (Lemma~\ref{lemma:offOnRel}).}
\end{proofsketch}

Some observations are in order. We first note a dependence on the term $1/(1-\overline{\rho})$, which, in turn, depends on the upper bound $\overline{\rho}$ of the spectral gap $\rho(\As)$. If the system does not display a dynamics, \ie we can set $\overline{\rho}=0$, we obtain a regret bound that, apart from logarithmic terms, coincides with that of \linucb, \ie $\widetilde{\mathcal{O}}(d \sigma \sqrt{T})$. Instead, for slow-converging systems, \ie $\overline{\rho} \approx 1$, the regret bound enlarges, as expected. Clearly, a value of $\overline{\rho}$ too large compared to the optimization horizon $T$ (\eg $\overline{\rho} = 1-1/T^{1/3}$) makes the regret bound degenerate to linear. This is a case in which the underlying system is so slow that the whole horizon $T$ is insufficient to approximately reach the steady state. Third, the regret bound is the sum of three components: the first one depends on the subgaussian proxy $\sigma$ and is due to the noisy estimation of the relevant quantities; the second one is a bias due to the epoch-based structure of \algnameshort; finally, the third one is constant (does not depend on $T$) accounts for the time needed to reach the steady state.

\begin{remark}[Regret upper bound (Theorem~\ref{thr:regretThr}) and lower bound (Theorem~\ref{thr:lb}) Comparison]
Apart from logarithmic terms, we notice a tight dependence on $d$ and on $T$. Instead, concerning the spectral properties of $\As$, in the upper bound, we experience a dependence on $1/(1-\overline{\rho})$ raised to a higher power (either $1$ for the term multiplied by $d$ and $3/2$ for the term multiplied by $\sqrt{d}$) \wrt the exponent appearing in the lower bound (\ie $1/2$). It is currently an open question whether the lower bound is not tight (which is obtained for a simpler setting in which the state is observable $\xs_t$) or whether more efficient algorithms for DLBs can be designed. Furthermore, Theorem~\ref{thr:regretThr} highlights the impact of the upper bound $\overline{\rho}$ compared with the true $\rho(\As)$.
\end{remark}
\section{Related Works}\label{sec:relatedworks}
In this section, we survey and compare the literature with a particular focus on bandits with delayed, aggregated, and composite feedback~\citep{JoulaniGS13} and online control for Linear Time-Invariant (LTI) systems~\citep{hespanha2018linear}. Additional related works are reported in Appendix~\ref{apx:relWorks}.

\textbf{Bandits with Delayed/Aggregated/Composite Feedback}~~ The Multi-Armed Bandit setting has been widely employed as a principled approach to address sequential decision-making problems~\citep{lattimore2020bandit}. The possibility of experiencing delayed rewards has been introduced by~\citet{JoulaniGS13} and widely exploited in advertising applications~\citep{olivier2014modeling, VernadeCP17}. A large number of approaches have extended this setting either considering stochastic delays~\citep{VernadeCLZEB20}, unknown delays~\citep{LiCG19, LancewickiSKM21}, arm-dependent delays~\citep{ManegueuVCV20}, non-stochastic delays~\citep{ItoHSTFKK20, ThuneCS19, abs-2201-13172}. Some methods relaxed the assumption that the individual reward is revealed after the delay expires, admitting the possibility of receiving anonymous feedback, which can be aggregated~\citep{Pike-Burke0SG18, abs-2112-13029} or composite~\citep{Cesa-BianchiGM18, abs-1910-01161, WangWH21}. Most of these approaches are able to achieve $\widetilde{\mathcal{O}}(\sqrt{T})$ regret, plus additional terms depending on the extent of the delay. In our DLBs, the reward is generated over time as a combined effect of past and present actions through a \emph{hidden state}, while these approaches generate the reward instantaneously and reveal it (individually or in aggregate) to the learner in the future and no underlying state dynamics is present.

\textbf{Online Control of Linear Time-Invariant Systems}~~The particular structure imposed by linear dynamics makes our approach comparable to LTI online control for partially observable systems~\citep[\eg][]{lale2020regret, simchowitz2020improper, plevrakis2020geometric}. While the dynamical model is similar, in online control of LTI systems, the perspective is quite different. Most of the works either consider the Linear Quadratic Regulator~\citep{mania2019certainty, lale2020regret} or (strongly) convex objective functions~\citep{mania2019certainty, simchowitz2020improper, lale2020logarithmic}, achieving, in most of the cases $\widetilde{\mathcal{O}} (\sqrt{T})$ regret for strongly convex functions and $\widetilde{\mathcal{O}} (T^{2/3})$ for convex functions. Recently,  $\widetilde{\mathcal{O}} (\sqrt{T})$ regret rate has been obtained for convex function too, by means of geometric exploration methods~\citep{plevrakis2020geometric}. Compared to \algnameshort, the algorithm of ~\citet{plevrakis2020geometric} considers general convex costs but assumes the observability of the state and limits to the class of disturbance response controllers~\citep{li1993robust} that do not include the constant policy. Moreover, the regret bound of~\citet{plevrakis2020geometric} differs from Theorem~\ref{thr:regretThr}, as it shows a cubic dependence on the system order\footnote{This holds for \emph{known} cost functions. Instead, for \emph{unknown} costs, the exponent becomes $24$~\citep{plevrakis2020geometric}.} and an implicit non-trivial dependence on the dynamic matrix $\As$. Instead, our Theorem~\ref{thr:regretThr} is remarkably independent of the system order $n$. Furthermore, \citet{lale2020logarithmic} reach $ \mathcal{O} (\log(T))$ regret in the case of strongly convex cost functions competing against the best \emph{persistently exciting} controller (\ie a controller implicitly maintaining a non-null exploration). Some approaches are designed to deal with adversarial noise~\citep{simchowitz2020improper}. All of these solutions, however, look for the best closed-loop controller within a specific class, \eg disturbance response control~\citep{li1993robust}. These controllers, however, do not allow us to easily incorporate constraints on the action space, which could be of crucial importance in practice, \eg in advertising domains. \algnameshort works with an arbitrary action space and, thanks to the linearity of the reward, does not require complex closed-loop controllers.
\section{Numerical Simulations}\label{sec:numericalsimulations}

In this section, we provide numerical validations of \algnameshort in both a synthetic scenario and a domain obtained from real-world data. The goal of these simulations is to highlight the behavior of \algnameshort in comparison with bandit baselines, describing advantages and disadvantages. 
The first experiment is a synthetic setting in which we can evaluate the performances of all the solutions and the sensitivity of \algnameshort w.r.t.~the $\overline{\rho}$ parameter (Section~\ref{subsec:exp_synt}). Then, we show a comparison in a DLB scenario retrieved from real-world data (Section~\ref{sec:realWorldExp}). 
The code of the experiments can be found at \url{https://github.com/marcomussi/DLB}. 
Details and additional experiments can be found in Appendix~\ref{apx:exper}.

\textbf{Baselines}~~We consider as main baseline \linucb \citep{abbasi2011improved}, designed for linear bandits. 
We include \expthree~\citep{auer1995gambling} usually employed in (non-adaptive) adversarial settings, and its extension to $k$-length memory (adaptive) adversaries \batchexpthree by~\citet{dekel2012online}.\footnote{$k$ is proportional to $\lfloor \log M / \log (1/\overline{\rho}) \rfloor$. In Appendix~\ref{apx:advApx} we elaborate on the use of adversarial bandit algorithms for DLBs.}
Additionally, we perform a comparison with algorithms for regret minimization in non-stationary environments: \dlinucb~\citep{RussacVC19}, an extension of \linucb for non-stationary settings, and \artwo~\citep{chen2021dynamic}, a bandit algorithm for processes presenting temporal structure. Lastly, in the case of real-world data, we compare our solution with a human-expert policy (\manualexpert). 
This policy is directly generalized from the original dataset by learning via regression the average budget allocation over all platforms from the available data.

For the baselines which do not support vectorial actions, we perform a discretization of the action space $\Us$ that surely contains optimal action. 
Concerning the hyperparameters of the baselines, whenever possible, they are selected as in the respective original papers. 
The experiments are presented with a regularization parameter $\lambda \in\{ 1,\log T$\} for the algorithms which require it (\ie \algnameshort, \linucb, and \dlinucb).\footnote{For \algnameshort, $\log T$ is a nearly optimal choice for $\lambda$ as it can be seen by looking at the first two addenda of the exploration factor in Equation~\eqref{eq:beta}.} 
Further information about the hyperparameters of the baselines and the adopted optimistic exploration bounds are presented in Appendix~\ref{apx:notes_on_baselines}.

\subsection{Synthetic Data}
\label{subsec:exp_synt}

\textbf{Setting}~~We consider a DLB defined by the following matrices $\As = \mathrm{diag}((0.2, 0, 0.1))$, $\Bs = \mathrm{diag}((0.25,0,0.1))$, $\vtheta = (0, 0.5, 0.1)^\transpose$, $\vomega=(1,0,0.1)^\transpose$ and a Gaussian noise with $\sigma = 0.01$ (diagonal covariance matrix for the state noise).\footnote{It is worth noting that the decision of using diagonal matrices is just for explanation purposes and w.l.o.g. (at least in the class of diagonalizable dynamic matrices). 
Indeed, we are just interested in the cumulative Markov parameter $\hs$ and we could have obtained the same results with an equivalent (non-diagonal) representation, by applying an inevitable transformation $\mathbf{T}$ as $\As' = \mathbf{T}\As\mathbf{T}^{-1}$, $\vomega' =\mathbf{T}^{-\transpose} \vomega $, and $\Bs' =\mathbf{T}\Bs $.} 
This way, the spectral gap of the dynamical matrix is $\rho(\As) = 0.2$ and $\Phi(\As) = 1$. Moreover, the cumulative Markov parameter is given by $\hs = (0.56, 0.5, 0.11)^\transpose$. 
We consider the action space $\Us = \{(u_1,u_2,u_3)^\transpose \in [0,1]^3 \text{ with } u_1+u_2+u_3 \le 1.5\}$ that simulates a total budget of $1.5$ to be allocated to the three platforms. Thus, a \quotes{myopic} agent would simply look at how the action immediately propagates to the reward through $\vtheta$, and will invest the budget in the second component of the action, which is weighted by $0.5$. 
Instead, a \quotes{far-sighted} agent, aware of the system evolution, will look at the cumulative Markov parameter $\hs$, realizing that the most convenient action is investing in the first component, weighted by $0.56$. Therefore, the optimal action is $\us^* = (1, 0.5, 0)^\transpose$ leading to $J^* = 0.81$.

\begin{figure*}[t!]
\minipage{0.3\textwidth}
\resizebox{\linewidth}{!}{\includegraphics[]{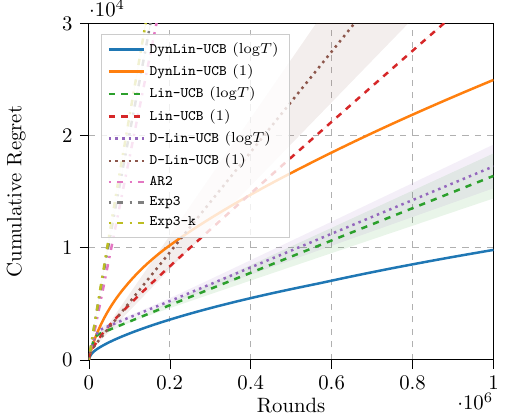}}
\caption{Cumulative regret as a function of the rounds comparing \algnameshort and the other bandit baselines (50 runs, mean $\pm$ std).\\}\label{fig:cum_regert}
\endminipage\hfill
\minipage{0.3\textwidth}
\resizebox{\linewidth}{!}{\includegraphics[]{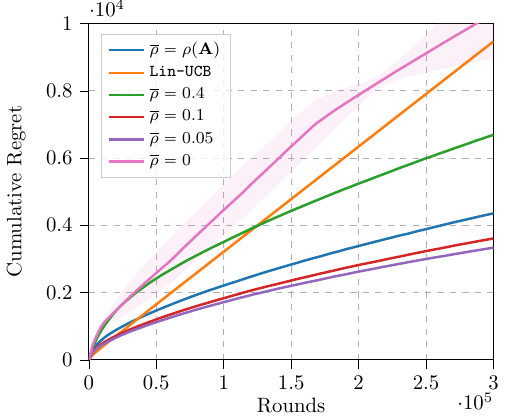}}
\caption{Cumulative regret as a function of the rounds comparing \linucb, and \algnameshort with $\lambda = \log T$, varying the upper bound on the spectral radius $\overline{\rho}$ (50 runs, mean $\pm$ std).}\label{fig:sensitivity}
\endminipage\hfill
\minipage{0.3\textwidth}
\resizebox{\linewidth}{!}{\includegraphics[]{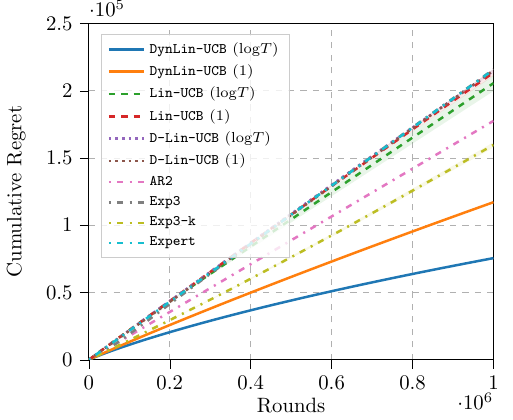}} \caption{Cumulative regret for \algnameshort, the other bandit baselines and the \manualexpert in the system generalized from real-world data (50 runs, mean $\pm$ std).}\label{fig:regret_realdata}
\endminipage
\end{figure*}

\textbf{Comparison with the bandit baselines}~~Figure~\ref{fig:cum_regert} shows the performance in terms of cumulative regret of \algnameshort, \linucb, \dlinucb, \artwo, \expthree, and \batchexpthree. The experiments are conducted over a time horizon of $1$ million rounds. For \algnameshort, we employed, for the sake of this experiment, the true value of the spectral gap, \ie $\overline{\rho} = \rho(\As)=0.2$.
First of all, we observe that both \expthree and \batchexpthree suffers a significantly large cumulative regret. Similar behavior is displayed by \artwo. 
Moreover, all the versions of \linucb and \dlinucb suffer linear regret. The best performance of \dlinucb is obtained when the discount factor $\gamma$ is close to $1$ (the weights take the form $w_t = \gamma^{-t}$), and the behavior is comparable with the one of \linucb. Even for a quite fast system ($\rho(\As) = 0.2$), ignoring the system dynamics, and the presence of the hidden state, has made both \linucb and \dlinucb commit (in their best version, with $\lambda = \log T$) to the sub-optimal (myopic) action $\us^{\circ} = (0.5,1,0)^\transpose$ with performance $J^{\circ}=0.78 < J^*$, with also a relevant variance. 
On the other hand, \algnameshort is able to maintain a smaller and stable (variance is negligible) sublinear regret in both its versions, with a notable advantage when using $\lambda = \log T$. 

\textbf{Sensitivity to the Choice of $\overline{\rho}$}~~The upper bound $\overline{\rho}$ of the spectral radius $\rho(\As)=0.2$ represents a crucial parameter of \algnameshort. While an overestimation $\overline{\rho} \gg \rho(\As)$ does not compromise the regret rate but tends to slow down the convergence process, a severe underestimation $\overline{\rho} \ll \rho(\As)$ might prevent learning at all. In Figure~\ref{fig:sensitivity}, we test \algnameshort against a misspecification of $\overline{\rho}$, when $\lambda = \log T$. We can see that by considering $\overline{\rho} = 2 \rho(\As)$, \algnameshort experiences a larger regret but still sublinear and smaller w.r.t. \linucb with $\lambda = \log T$. Even by reducing $\overline{\rho} \in \{0.1,0.05\}$, \algnameshort is able to keep the regret sublinear, showing remarkable robustness to misspecification. Clearly, setting $\overline{\rho}=0$ makes the regret almost degenerate to linear.

\subsection{Real-world Data}\label{sec:realWorldExp}
We present an experimental evaluation based on real-world data coming from three web advertising platforms (\texttt{Facebook}, \texttt{Google}, and \texttt{Bing}), related to several campaigns for an invested budget of $5$ Million EUR over $2$ years. Starting from such data, we learn the best DLB model by means of a specifically designed variant of the Ho-Kalman algorithm~\citep{ho1966effective}.\footnote{See Appendix~\ref{apx:systemidenfication}.} We used the learned model to build up a simulator. The resulting system has $\rho (\As) = 0.67$. We evaluate \algnameshort against the baselines for $T=10^{6}$ steps over $50$ runs.

\textbf{Results}~~Figure~\ref{fig:regret_realdata} shows the results in terms of cumulative regret. It is worth noting that no algorithm, except for \algnameshort, is able to converge to the optimal choice. Indeed, they immediately commit to a sub-optimal solution. \algnameshort, instead, shows a convergence trend towards the optimal policy over time for both $\lambda=1$ and $\lambda=\log T$, even if the best-performing version is the one which employs $\lambda=\log T$. The \manualexpert, which has a preference towards maximizing the instantaneous effect of the actions only and does not take into account correlations between platforms, displays a sub-optimal performance.

\section{Discussion and Conclusions}
\label{sec:conclusions}
In this paper, we have introduced the \settingName (DLBs), a novel model to represent sequential decision-making problems in which the system is characterized by a non-observable hidden state that evolves according to linear dynamics and by an observable noisy reward that linearly combines the hidden state and the action played. This model accounts for scenarios that cannot be easily represented by existing bandit models that consider delayed and aggregated feedback. We have derived a regret lower bound that highlights the main complexities of the DLB problem. Then, we have proposed a novel optimistic regret minimization approach, \algnameshort, that, under stability assumption, is able to achieve sub-linear regret. The numerical simulation in both synthetic and real-world domains succeeded in showing that, in a setting where the baselines mostly suffer linear regret, our algorithm consistently enjoys sublinear regret. Furthermore, \algnameshort proved to be robust to misspecification of its most relevant hyper-parameter $\overline{\rho}$. To the best of our knowledge, this is the first work addressing this family of problems, characterized by hidden linear dynamics, with a simple, yet effective, bandit-like approach. 
Short-term future directions include efforts in closing the gap between the regret lower and upper bounds.
Long-term future directions should focus on extending the present approach to non-linear system dynamics and embedding in the algorithm additional budget constraints enforced over the optimization horizon.

\section*{Acknowledgements}
This paper is supported by PNRR-PE-AI FAIR project funded by the NextGeneration EU program.

\bibliography{biblio}

\begin{thebibliography}{48}
\providecommand{\natexlab}[1]{#1}
\providecommand{\url}[1]{\texttt{#1}}
\expandafter\ifx\csname urlstyle\endcsname\relax
  \providecommand{\doi}[1]{doi: #1}\else
  \providecommand{\doi}{doi: \begingroup \urlstyle{rm}\Url}\fi

\bibitem[Abbasi{-}Yadkori \& Szepesv{\'{a}}ri(2011)Abbasi{-}Yadkori and
  Szepesv{\'{a}}ri]{abbasi2011regret}
Abbasi{-}Yadkori, Y. and Szepesv{\'{a}}ri, C.
\newblock Regret bounds for the adaptive control of linear quadratic systems.
\newblock In \emph{The 24th Annual Conference on Learning Theory}, pp.\  1--26,
  2011.

\bibitem[Abbasi{-}Yadkori et~al.(2011)Abbasi{-}Yadkori, P{\'{a}}l, and
  Szepesv{\'{a}}ri]{abbasi2011improved}
Abbasi{-}Yadkori, Y., P{\'{a}}l, D., and Szepesv{\'{a}}ri, C.
\newblock Improved algorithms for linear stochastic bandits.
\newblock In \emph{Advances in Neural Information Processing Systems}, pp.\
  2312--2320, 2011.

\bibitem[Agarwal et~al.(2019)Agarwal, Hazan, and Singh]{AgarwalHS19}
Agarwal, N., Hazan, E., and Singh, K.
\newblock Logarithmic regret for online control.
\newblock In \emph{Advances in Neural Information Processing Systems}, pp.\
  10175--10184, 2019.

\bibitem[{\AA}str{\"o}m(1965)]{aastrom1965optimal}
{\AA}str{\"o}m, K.~J.
\newblock Optimal control of markov processes with incomplete state
  information.
\newblock \emph{Journal of mathematical analysis and applications}, 10\penalty0
  (1):\penalty0 174--205, 1965.

\bibitem[Auer et~al.(1995)Auer, Cesa-Bianchi, Freund, and
  Schapire]{auer1995gambling}
Auer, P., Cesa-Bianchi, N., Freund, Y., and Schapire, R.~E.
\newblock Gambling in a rigged casino: The adversarial multi-armed bandit
  problem.
\newblock In \emph{Proceedings of IEEE 36th annual foundations of computer
  science}, pp.\  322--331. IEEE, 1995.

\bibitem[Auer et~al.(2002)Auer, Cesa{-}Bianchi, Freund, and
  Schapire]{AuerCFS02}
Auer, P., Cesa{-}Bianchi, N., Freund, Y., and Schapire, R.~E.
\newblock The nonstochastic multiarmed bandit problem.
\newblock \emph{{SIAM} J. Comput.}, 32\penalty0 (1):\penalty0 48--77, 2002.

\bibitem[Bacchiocchi et~al.(2022)Bacchiocchi, Genalti, Maran, Mussi, Restelli,
  Gatti, and Metelli]{bacchiocchi2022autoregressive}
Bacchiocchi, F., Genalti, G., Maran, D., Mussi, M., Restelli, M., Gatti, N.,
  and Metelli, A.~M.
\newblock Autoregressive bandits.
\newblock \emph{CoRR}, abs/2212.06251, 2022.

\bibitem[Berman(2018)]{berman2018beyond}
Berman, R.
\newblock Beyond the last touch: Attribution in online advertising.
\newblock \emph{Marketing Science}, 37\penalty0 (5):\penalty0 771--792, 2018.

\bibitem[Cesa{-}Bianchi et~al.(2018)Cesa{-}Bianchi, Gentile, and
  Mansour]{Cesa-BianchiGM18}
Cesa{-}Bianchi, N., Gentile, C., and Mansour, Y.
\newblock Nonstochastic bandits with composite anonymous feedback.
\newblock In \emph{Conference On Learning Theory}, pp.\  750--773, 2018.

\bibitem[Chapelle(2014)]{olivier2014modeling}
Chapelle, O.
\newblock Modeling delayed feedback in display advertising.
\newblock In \emph{Proceedings of the 20th ACM SIGKDD International Conference
  on Knowledge Discovery and Data Mining}, pp.\  1097–1105. Association for
  Computing Machinery, 2014.

\bibitem[Chen et~al.(2021)Chen, Golrezaei, and Bouneffouf]{chen2021dynamic}
Chen, Q., Golrezaei, N., and Bouneffouf, D.
\newblock Dynamic bandits with temporal structure.
\newblock \emph{Available at SSRN 3887608}, 2021.

\bibitem[Court et~al.(2009)Court, Elzinga, Mulder, and
  Vetvik]{court2009consumer}
Court, D., Elzinga, D., Mulder, S., and Vetvik, O.~J.
\newblock The consumer decision journey.
\newblock \emph{McKinsey Quarterly}, 3:\penalty0 96--107, 2009.

\bibitem[Dekel et~al.(2012)Dekel, Tewari, and Arora]{dekel2012online}
Dekel, O., Tewari, A., and Arora, R.
\newblock Online bandit learning against an adaptive adversary: from regret to
  policy regret.
\newblock In \emph{International Conference on Machine Learning}, 2012.

\bibitem[Garg \& Akash(2019)Garg and Akash]{abs-1910-01161}
Garg, S. and Akash, A.~K.
\newblock Stochastic bandits with delayed composite anonymous feedback.
\newblock \emph{CoRR}, abs/1910.01161, 2019.

\bibitem[Gur et~al.(2014)Gur, Zeevi, and Besbes]{GurZB14}
Gur, Y., Zeevi, A., and Besbes, O.
\newblock Stochastic multi-armed-bandit problem with non-stationary rewards.
\newblock In \emph{Advances in Neural Information Processing Systems}, pp.\
  199--207, 2014.

\bibitem[Hespanha(2018)]{hespanha2018linear}
Hespanha, J.~P.
\newblock \emph{Linear Systems Theory: Second Edition}.
\newblock Princeton University Press, 2018.

\bibitem[Ho \& Kalman(1966)Ho and Kalman]{ho1966effective}
Ho, B.~L. and Kalman, R.~E.
\newblock Effective construction of linear state-variable models from
  input/output functions.
\newblock \emph{at-Automatisierungstechnik}, 14\penalty0 (1-12):\penalty0
  545--548, 1966.

\bibitem[Hoban \& Bucklin(2015)Hoban and Bucklin]{hoban2015effects}
Hoban, P.~R. and Bucklin, R.~E.
\newblock Effects of internet display advertising in the purchase funnel:
  Model-based insights from a randomized field experiment.
\newblock \emph{Journal of Marketing Research}, 52\penalty0 (3):\penalty0
  375--393, 2015.

\bibitem[Hsu et~al.(2012)Hsu, Kakade, and Zhang]{hsu2012tail}
Hsu, D., Kakade, S., and Zhang, T.
\newblock A tail inequality for quadratic forms of subgaussian random vectors.
\newblock \emph{Electronic Communications in Probability}, 17:\penalty0 1--6,
  2012.

\bibitem[Isom et~al.(2008)Isom, Meyn, and Braatz]{isom2008piecewise}
Isom, J.~D., Meyn, S.~P., and Braatz, R.~D.
\newblock Piecewise linear dynamic programming for constrained pomdps.
\newblock In \emph{Proceedings of the Twenty-Third {AAAI} Conference on
  Artificial Intelligence}, pp.\  291--296. {AAAI} Press, 2008.

\bibitem[Ito et~al.(2020)Ito, Hatano, Sumita, Takemura, Fukunaga, Kakimura, and
  Kawarabayashi]{ItoHSTFKK20}
Ito, S., Hatano, D., Sumita, H., Takemura, K., Fukunaga, T., Kakimura, N., and
  Kawarabayashi, K.
\newblock Delay and cooperation in nonstochastic linear bandits.
\newblock In \emph{Advances in Neural Information Processing Systems}, 2020.

\bibitem[Jin et~al.(2022)Jin, Lancewicki, Luo, Mansour, and
  Rosenberg]{abs-2201-13172}
Jin, T., Lancewicki, T., Luo, H., Mansour, Y., and Rosenberg, A.
\newblock Near-optimal regret for adversarial {MDP} with delayed bandit
  feedback.
\newblock \emph{CoRR}, abs/2201.13172, 2022.

\bibitem[Joulani et~al.(2013)Joulani, Gy{\"{o}}rgy, and
  Szepesv{\'{a}}ri]{JoulaniGS13}
Joulani, P., Gy{\"{o}}rgy, A., and Szepesv{\'{a}}ri, C.
\newblock Online learning under delayed feedback.
\newblock In \emph{Proceedings of the 30th International Conference on Machine
  Learning}, pp.\  1453--1461, 2013.

\bibitem[Kalman(1963)]{kalman1963mathematical}
Kalman, R.~E.
\newblock Mathematical description of linear dynamical systems.
\newblock \emph{Journal of the Society for Industrial and Applied Mathematics,
  Series A: Control}, 1\penalty0 (2):\penalty0 152--192, 1963.

\bibitem[Kim et~al.(2011)Kim, Lee, Kim, and Poupart]{kim2011point}
Kim, D., Lee, J., Kim, K., and Poupart, P.
\newblock Point-based value iteration for constrained pomdps.
\newblock In \emph{Proceedings of the 22nd International Joint Conference on
  Artificial Intelligence}, pp.\  1968--1974, 2011.

\bibitem[Lale et~al.(2020{\natexlab{a}})Lale, Azizzadenesheli, Hassibi, and
  Anandkumar]{lale2020logarithmic}
Lale, S., Azizzadenesheli, K., Hassibi, B., and Anandkumar, A.
\newblock Logarithmic regret bound in partially observable linear dynamical
  systems.
\newblock In \emph{Advances in Neural Information Processing Systems},
  2020{\natexlab{a}}.

\bibitem[Lale et~al.(2020{\natexlab{b}})Lale, Azizzadenesheli, Hassibi, and
  Anandkumar]{lale2020regret}
Lale, S., Azizzadenesheli, K., Hassibi, B., and Anandkumar, A.
\newblock Regret minimization in partially observable linear quadratic control.
\newblock \emph{CoRR}, abs/2002.00082, 2020{\natexlab{b}}.

\bibitem[Lancewicki et~al.(2021)Lancewicki, Segal, Koren, and
  Mansour]{LancewickiSKM21}
Lancewicki, T., Segal, S., Koren, T., and Mansour, Y.
\newblock Stochastic multi-armed bandits with unrestricted delay distributions.
\newblock In \emph{Proceedings of the 38th International Conference on Machine
  Learning}, pp.\  5969--5978, 2021.

\bibitem[Lattimore \& Szepesv{\'a}ri(2020)Lattimore and
  Szepesv{\'a}ri]{lattimore2020bandit}
Lattimore, T. and Szepesv{\'a}ri, C.
\newblock \emph{Bandit algorithms}.
\newblock Cambridge University Press, 2020.

\bibitem[Li et~al.(2019)Li, Chen, and Giannakis]{LiCG19}
Li, B., Chen, T., and Giannakis, G.~B.
\newblock Bandit online learning with unknown delays.
\newblock In \emph{The 22nd International Conference on Artificial Intelligence
  and Statistics}, pp.\  993--1002, 2019.

\bibitem[Li \& Bosch(1993)Li and Bosch]{li1993robust}
Li, H.~X. and Bosch, P. P. J. V.~D.
\newblock A robust disturbance-based control and its application.
\newblock \emph{International Journal of Control}, 58\penalty0 (3):\penalty0
  537--554, 1993.

\bibitem[Manegueu et~al.(2020)Manegueu, Vernade, Carpentier, and
  Valko]{ManegueuVCV20}
Manegueu, A.~G., Vernade, C., Carpentier, A., and Valko, M.
\newblock Stochastic bandits with arm-dependent delays.
\newblock In \emph{Proceedings of the 37th International Conference on Machine
  Learning}, pp.\  3348--3356, 2020.

\bibitem[Mania et~al.(2019)Mania, Tu, and Recht]{mania2019certainty}
Mania, H., Tu, S., and Recht, B.
\newblock Certainty equivalence is efficient for linear quadratic control.
\newblock In \emph{Advances in Neural Information Processing Systems}, pp.\
  10154--10164, 2019.

\bibitem[Nobari(2019)]{Nobari19}
Nobari, S.
\newblock {DBA:} dynamic multi-armed bandit algorithm.
\newblock In \emph{The Thirty-Third {AAAI} Conference on Artificial
  Intelligence}, pp.\  9869--9870, 2019.

\bibitem[Oymak \& Ozay(2019)Oymak and Ozay]{oymak2O19nonasymptotic}
Oymak, S. and Ozay, N.
\newblock Non-asymptotic identification of {LTI} systems from a single
  trajectory.
\newblock In \emph{2019 American Control Conference}, pp.\  5655--5661, 2019.

\bibitem[Pike{-}Burke et~al.(2018)Pike{-}Burke, Agrawal, Szepesv{\'{a}}ri, and
  Gr{\"{u}}new{\"{a}}lder]{Pike-Burke0SG18}
Pike{-}Burke, C., Agrawal, S., Szepesv{\'{a}}ri, C., and
  Gr{\"{u}}new{\"{a}}lder, S.
\newblock Bandits with delayed, aggregated anonymous feedback.
\newblock In \emph{Proceedings of the 35th International Conference on Machine
  Learning}, pp.\  4102--4110, 2018.

\bibitem[Plevrakis \& Hazan(2020)Plevrakis and Hazan]{plevrakis2020geometric}
Plevrakis, O. and Hazan, E.
\newblock Geometric exploration for online control.
\newblock In \emph{Advances in Neural Information Processing Systems}, 2020.

\bibitem[Russac et~al.(2019)Russac, Vernade, and Capp{\'{e}}]{RussacVC19}
Russac, Y., Vernade, C., and Capp{\'{e}}, O.
\newblock Weighted linear bandits for non-stationary environments.
\newblock In \emph{Advances in Neural Information Processing Systems}, pp.\
  12017--12026, 2019.

\bibitem[Sarkar et~al.(2021)Sarkar, Rakhlin, and Dahleh]{sarkar2021finite}
Sarkar, T., Rakhlin, A., and Dahleh, M.~A.
\newblock Finite time {LTI} system identification.
\newblock \emph{J. Mach. Learn. Res.}, 22:\penalty0 26:1--26:61, 2021.

\bibitem[Simchowitz et~al.(2020)Simchowitz, Singh, and
  Hazan]{simchowitz2020improper}
Simchowitz, M., Singh, K., and Hazan, E.
\newblock Improper learning for non-stochastic control.
\newblock In \emph{Conference on Learning Theory}, volume 125, pp.\
  3320--3436. {PMLR}, 2020.

\bibitem[Thune et~al.(2019)Thune, Cesa{-}Bianchi, and Seldin]{ThuneCS19}
Thune, T.~S., Cesa{-}Bianchi, N., and Seldin, Y.
\newblock Nonstochastic multiarmed bandits with unrestricted delays.
\newblock In \emph{Advances in Neural Information Processing Systems}, pp.\
  6538--6547, 2019.

\bibitem[Tsiamis \& Pappas(2019)Tsiamis and Pappas]{tsiamis2019finite}
Tsiamis, A. and Pappas, G.~J.
\newblock Finite sample analysis of stochastic system identification.
\newblock In \emph{58th {IEEE} Conference on Decision and Control}, pp.\
  3648--3654, 2019.

\bibitem[Undurti \& How(2010)Undurti and How]{undurti2010online}
Undurti, A. and How, J.~P.
\newblock An online algorithm for constrained pomdps.
\newblock In \emph{{IEEE} International Conference on Robotics and Automation},
  pp.\  3966--3973. {IEEE}, 2010.

\bibitem[Vernade et~al.(2017)Vernade, Capp{\'{e}}, and Perchet]{VernadeCP17}
Vernade, C., Capp{\'{e}}, O., and Perchet, V.
\newblock Stochastic bandit models for delayed conversions.
\newblock In \emph{Proceedings of the Thirty-Third Conference on Uncertainty in
  Artificial Intelligence}, 2017.

\bibitem[Vernade et~al.(2020)Vernade, Carpentier, Lattimore, Zappella, Ermis,
  and Br{\"{u}}ckner]{VernadeCLZEB20}
Vernade, C., Carpentier, A., Lattimore, T., Zappella, G., Ermis, B., and
  Br{\"{u}}ckner, M.
\newblock Linear bandits with stochastic delayed feedback.
\newblock In \emph{Proceedings of the 37th International Conference on Machine
  Learning}, pp.\  9712--9721, 2020.

\bibitem[Wang et~al.(2021)Wang, Wang, and Huang]{WangWH21}
Wang, S., Wang, H., and Huang, L.
\newblock Adaptive algorithms for multi-armed bandit with composite and
  anonymous feedback.
\newblock In \emph{Thirty-Fifth {AAAI} Conference on Artificial Intelligence},
  pp.\  10210--10217, 2021.

\bibitem[Zhang et~al.(2021)Zhang, Tsuchida, and Ong]{abs-2112-13029}
Zhang, M., Tsuchida, R., and Ong, C.~S.
\newblock Gaussian process bandits with aggregated feedback.
\newblock \emph{CoRR}, abs/2112.13029, 2021.

\bibitem[Åström(1965)]{astrom1965optimal}
Åström, K.
\newblock Optimal control of markov processes with incomplete state
  information.
\newblock \emph{Journal of Mathematical Analysis and Applications}, 10\penalty0
  (1):\penalty0 174--205, 1965.

\end{thebibliography}
\bibliographystyle{icml2023}

\newpage
\appendix
\onecolumn
\section{Additional Related Works}\label{apx:relWorks}
In this appendix, we report additional details about the related works.

\subsection{Delayed/Aggregated Feedback with DLBs}\label{apx:delComp}
In this appendix, we show how we can model \emph{delayed} and \emph{composite} feedback with DLBs. For the delayed feedback, we focus on the case in which either the delay is fixed to the value $\tau \ge 1$, \ie the reward of the pull performed at round $t$ is experienced at round $t+\tau$. For the composite feedback, we assume that the reward of the pull performed at round $t$ is spread over the next $\tau \ge 1$ rounds with fixed weights $(w_1,\dots,w_\tau)$. Denoting with $R_t$ the full reward (not observed) due to the pull performed at round $t$, the agent  at round $t$ observes the weighted sum of the rewards reported below:\footnote{It is worth noting that the fixed-delay case is a particular case of composite feedback, where $w_1=\dots=w_{\tau-1}=0$ and $w_\tau=1$.}
\begin{align}
    \sum_{l=1}^{\tau} w_l R_{t-l}.
\end{align}
These two cases can be modeled as DLBs with a suitable encoding of the arms and choice of matrices. In particular, assuming to have $K$ arms, we take the arm set $\mathcal{U}$ to be the canonical basis of $\Reals^K$, and we denote with $\bm{\mu}$ the vector of expected rewards.
We define $\boldsymbol{\theta} = \boldsymbol{0}$ and:
$$
    \As = \begin{pmatrix}
                    0 & 0 & \dots & 0 & 0 \\
                    1 & 0 & \dots & 0 & 0 \\
                    0 & 1 & \dots & 0 & 0\\
                    \vdots & \vdots & \ddots & \vdots & \vdots \\
                    0 & 0 & \dots & 1 & 0 \\
                    \end{pmatrix} \in \mathbb{R}^{\tau \times \tau},
    \qquad
    \Bs = \begin{pmatrix}
                    \bm{\mu}^T_K \\ \bm{0}^T_K \\ \bm{0}^T_K \\ \vdots \\ \bm{0}^T_K
                    \end{pmatrix} \in \Reals^{\tau \times K},
$$
$$
    \boldsymbol{\omega}_{\text{delay}} = \begin{pmatrix}
                            0 \\ 0 \\ 0 \\ \vdots \\ 1
                            \end{pmatrix} \in \Reals^{\tau},
    \qquad
    \boldsymbol{\omega}_{\text{composite}} = \begin{pmatrix}
                            w_1 \\ w_2 \\ w_3 \\ \vdots \\ w_{\tau}
                            \end{pmatrix} \in \Reals^{\tau}.
$$

However, DLBs cannot model random or adversarial delays. Nevertheless, DLBs can capture scenarios of composite feedback in which the reward is spread over an infinite number of rounds. Keeping the $K$-armed case introduced above, we can consider the simplest example of a reward that spreads as an autoregressive process AR(1) with parameter $\gamma \in (0,1)$, that cannot be represented using the standard composite feedback. In such a case, we simply need a system with order $n=1$ with matrices (actually scalars):
\begin{align*}
    \As = \gamma, \qquad \Bs = \bm{\us}^T, \qquad \vomega = 1.
\end{align*}
Clearly, one can consider AR($m$) processes~\citep{bacchiocchi2022autoregressive} by employing systems of order $n = m > 1$.

\subsection{Partially Observable Markov Decision Processes}
As already noted, looking at DLBs in their generality, we realize that our model is a particular subclass of the Partially Observable Markov Decision Processes~\citep[POMDP,][]{astrom1965optimal}. However, in the POMDP literature, no particular structure of the hidden state dynamics is assumed. The specific linear dynamics are rarely considered, as well as the possibility of a reward that is a linear combination of the hidden state and the action. Nevertheless, several works accounted for the presence of constraints~\citep{isom2008piecewise,undurti2010online,kim2011point} without exploiting the linearity and without regret guarantees.

\subsection{Adversarial Bandits}\label{apx:advApx}
It is worth elaborating on the adaptation of adversarial MAB algorithms to this setting. First, since the reward distribution in DLBs depends at every round $t$ on the sequence of actions played by the agent prior to $t$, we can reduce the DLB setting to an adversarial bandit with an \emph{adaptive} (or non-oblivious) adversary. Second, such an adversary must have \emph{infinite memory} in principle. Third, our regret definition of Section~\ref{sec:problemformulation} is a \emph{policy regret}~\citep{dekel2012online} that compares the algorithm performance against playing the optimal policy in hindsight from the beginning, as opposed to the \emph{external regret} often employed for non-adaptive adversaries. It is well known that for infinite-memory adaptive adversaries, no algorithm can achieve sublinear policy regret. Nevertheless, for \settingNameShort setting, we know that the effect of the past is always vanishing (given Assumption~\ref{ass:spectralNorm} enforcing $\rho (\As) < 1$), so we can approximate our setting as a \emph{finite-memory} setting, by considering memory length $k \propto \lceil \frac{\log{M}}{\log{1/\overline{\rho}}} \rceil$, where $M$ is the one defined in Algorithm~\ref{alg:alg}~(line~\ref{line:define_M}), with an additional regret term only logarithmic in the optimization horizon $T$. Then, given this approximation, we can make use of an adversarial bandit algorithm (designed for non-adaptive adversaries) in the framework proposed by~\citet{dekel2012online} to make it effective for the finite-memory adaptive adversary setting. In the case of an optimal algorithm, such as \expthree~\cite{AuerCFS02}, suffering an external regret of order $\widetilde{\mathcal{O}} (\sqrt{MT})$, being $M$ the number of arms, the version to address this  finite-memory adaptive adversary setting suffers a regret bounded by $\widetilde{\mathcal{O}} ((k+1)M^{1/3}T^{2/3})$, as shown in Theorem 2 of~\citet{dekel2012online}.

\subsection{Other Approaches}
Non-stationary bandits~\citep{GurZB14} can be regarded as bandits with a hidden state that evolves through a (possibly non-linear) dynamics. The main difference compared with our DLBs is that the hidden state evolves in an \emph{uncontrollable} way, \ie it does not depend on the sequence of actions performed so far. \citet{RussacVC19} extend the linear bandit setting by considering a non-stationary evolution of the parameter $\vtheta^*_t$. The notion of \emph{dynamic} bandit is further studied by~\citet{chen2021dynamic}, where an auto-regressive process is considered for the evolution of the reward through time and by~\citet{Nobari19} that propose a practical approach to cope with this setting.

\section{Proofs and Derivations}\label{apx:proofs_all_paper}
In this section, we provide the proofs we have omitted in the main paper.

\subsection{Proofs of Section~\ref{sec:problemformulation}}\label{proof:0001}

Before we proceed, we introduce a different notion of regret useful for analysis purposes, that we name \emph{offline regret}. This notion of regret compares $J^*$ with the steady-state performance of the action $\us_t = \vpi_t(H_{t-1})$ played at each round $t \in \dsb{T}$ by the agent: 
\begin{equation}
R^{\text{off}}(\underline{\vpi},T) \coloneqq T  J^* -  \sum_{t=1}^T J(\us_t).
\end{equation}
We denote with $\mathbb{E}R^{\text{off}}(\underline{\vpi},T)$ the \emph{expected offline regret}, where the expectation is taken \wrt the randomness of the reward.
Clearly, the two notions of regret coincide when the system has no dynamics.

The following result relates the offline and the (online) expected regret.

\begin{lemma}\label{lemma:offOnRel}
Under Assumptions~\ref{ass:spectralNorm} and~\ref{ass:boundedness}, for any policy $\bm{\underline{\pi}}$, it holds that:
\begin{align*}
	\left| \E R^{\text{off}}(\bm{\underline{\pi}}, T) - \E R(\bm{\underline{\pi}}, T) \right| \le \frac{\Omega \Phi(\As) B U}{(1-\rho(\As))^2} + \frac{\Omega \Phi(\As) X }{1-\rho(\As)}.
\end{align*}
\end{lemma}

\begin{proof}
First of all, we observe that for any policy, the cumulative effect of the noise components is zero-mean. Thus, it
suffices to consider the deterministic evolution of the system. For every $t \in \dsb{T}$, let us denote with $\E[y_t]$ the expected reward at time $t$ and with $J(\us_t)$ as the steady-state performance when executing action $\us_t$:

	\begin{align*}
		& \E[y_t] = \sum_{s=0}^{t-1} \inner{\hs^{\{s\}}}{ \us_{t-s}} + \vomega^\transpose \As^{t-1} \xs_1 = \vtheta^\transpose \us_t + \vomega^\transpose \sum_{s=1}^{t-1} \As^{s-1} \Bs \us_{t-s} + \vomega^\transpose \As^{t-1} \xs_1, \\
		& J(\us_t) = \vtheta^\transpose \us_t + \vomega^\transpose \left(\Is_d  - \As\right)^{-1} \us_t = \vtheta^\transpose \us_t + \vomega^\transpose \sum_{s=0}^{+\infty} \As^s \us_t.
	\end{align*}
	We now proceed by summing over $t \in \dsb{T}$. First of all, we consider the following preliminary result involving $y_t$, which is obtained by rearranging the summations:
	\begin{align*}
		\sum_{t=1}^T \E[y_t] & = \vtheta^\transpose \sum_{t=1}^T \us_t + \vomega^\transpose \sum_{t=1}^T \sum_{s=1}^{t-1} \As^{s-1} \Bs \us_{t-s} + \vomega^\transpose \sum_{t=1}^T \As^{t-1} \xs_1 \\
		&  = \vtheta^\transpose \sum_{t=1}^T \us_t + \vomega^\transpose \sum_{t=1}^{T-1} \left( \sum_{s=0}^{T-t-1} \As^s \right) \Bs \us_{t} + \vomega^\transpose \sum_{t=1}^T \As^{t-1} \xs_1.
	\end{align*}
	Thus, we have:
	\begin{align}
		\left|\sum_{t=1}^T  \left(J(\us_t) - \E[y_t] \right)\right| & = \left|\vomega^\transpose \sum_{t=1}^T \left( \sum_{s=0}^{+\infty} \As^{s} -  \sum_{s=0}^{T-t-1} \As^s \right) \Bs \us_{t} -  \vomega^\transpose \sum_{t=1}^T \As^{t-1} \xs_1 \right| \notag \\
		& = \left|\vomega^\transpose \sum_{t=1}^T \left( \sum_{s=T-t}^{+\infty} \As^s \right) \Bs \us_{t} - \vomega^\transpose \sum_{t=1}^T \As^{t-1} \xs_1 \right|\notag \\
		& \le \Omega \Phi(\As) B U \sum_{t=1}^T \sum_{s=T-t}^{+\infty}  \rho(\As)^{s} + \Omega \Phi(\As) X \sum_{t=1}^T \rho(\As)^{t-1}\label{ll:-1}\\
		& \le \frac{\Omega \Phi(\As) B U}{1-\rho(\As)} \sum_{t=1}^T \rho(\As)^{T-t} + \frac{\Omega \Phi(\As) X }{1-\rho(\As)}\label{ll:-3}\\
		& \le \frac{\Omega \Phi(\As) B U}{(1-\rho(\As))^2} + \frac{\Omega \Phi(\As) X }{1-\rho(\As)},\label{ll:-4}
	\end{align}
	where line~\eqref{ll:-1} follows from Assumptions~\ref{ass:spectralNorm} and~\ref{ass:boundedness}, lines~\eqref{ll:-3} and~\eqref{ll:-4} follow from bounding the summations with the series. The result follows by observing that:
	\begin{align*}
	\E R^{\text{off}}(\bm{\underline{\pi}}, T) - \E R(\bm{\underline{\pi}}, T)  = \sum_{t=1}^T  \left(J(\us_t) - \E[y_t] \right).
	\end{align*}
\end{proof}

\lb*

\begin{proof}
	To derive the lower bound, we take inspiration from the construction of~\cite{lattimore2020bandit} for linear bandits (Theorem 24.1). We consider a class of DLBs defined in terms of fixed $0 \le \rho < 1$ and $0 \le \epsilon \le \rho$ with $\vomega = \mathbf{1}_d$, $\vtheta = - \frac{2(1-\rho) + \epsilon}{2(1-(\rho-\epsilon))} \mathbf{1}_d$, $\Bs = (1-\rho)\Is_{d}$ and with a diagonal dynamical matrix $\As = \mathrm{diag}(\mathbf{a})$, defined in terms of the vector $\mathbf{a}$ belonging to the set $\mathcal{A} = \{\rho,\rho-\epsilon\}^d$.	The available actions are $\Us = \{-1,1\}^d$. Let us note that $|\mathcal{A}| = |\Us| = 2^d$. Thus, in our set of DLBs, the vector $\mathbf{a}$ fully characterizes the problem. Moreover, we observe that, given the diagonal $\mathbf{a} = \mathrm{diag}(\mathbf{A})$, we can compute the cumulative Markov parameter $\hs_{\mathbf{a}} = \mathrm{sign}(\mathbf{a}) \frac{\epsilon}{2(1-(\rho-\epsilon))}$.\footnote{For a vector $\mathbf{v} \in \Reals^d$, we denote with  $\mathrm{sign}(\mathbf{v}) \in \{-1,1\}^d$ the vector of the signs of the components of $\mathbf{v}$. It is irrelevant how we convene to define the sign of $0$.} As a consequence the optimal action can be defined as $\us_{\mathbf{a}}^* = \mathrm{sign}(\mathbf{a})$, whose performance is given by $J^*_{\mathbf{a}} = \inner{\hs_{\mathbf{a}}}{\us_{\mathbf{a}}^*} = \frac{\epsilon d}{2(1-(\rho-\epsilon))}$.
	
	Let us consider the probability distribution over the canonical bandit model induced by executing a policy $\bm{\underline{\pi}}$ in a DLB characterized by the diagonal of the dynamical matrix $\mathbf{a} \in \mathcal{A}$ and with Gaussian diagonal noise:
	\begin{align*}
		\mathbb{P}_{\mathbf{a}} = \prod_{t=1}^{T} \mathcal{N}(\xs_{t+1}| \As \xs_t + \Bs \us_t, \sigma^2 \Is_d) \mathcal{N} (y_t| \inner{\vtheta}{\us_t} + \inner{\vomega}{\xs_t}, \sigma^2) \pi_t(\us_t | H_{t-1}),
	\end{align*}
 where $H_{t-1}$ is the history of observations up to time $t-1$. We denote with $\mathbb{E}_{\mathbf{a}}$ the expectation induced by the distribution $\mathbb{P}_{\mathbf{a}}$.	For every $i \in \dsb{d}$, let us now consider an alternative DLB instance that differs on the dynamical matrix only. Specifically:
	 \begin{align*}
		\mathbf{a}'_j = \begin{cases}
			\mathbf{a}_j & \text{if } j \neq i \\
			\rho &  \text{if } j = i \text{ and } \mathbf{a}_j = \rho-\epsilon \\
			\rho-\epsilon & \text{if } j=i \text{ and } \mathbf{a}_j = \rho
		\end{cases}, \qquad \forall j \in \dsb{d}.
	\end{align*}
	By relative entropy identities~\citep{lattimore2020bandit}, let $\As = \mathrm{diag}(\mathbf{a})$ and $\As' = \mathrm{diag}(\mathbf{a}')$, we have:
	\begin{align*}
		D_{\text{KL}} \left( \mathbb{P}_{\mathbf{a}} , \mathbb{P}_{\mathbf{a}'} \right) & = \mathbb{E}_{\mathbf{a}} \left[ \sum_{t=1}^T D_{\text{KL}} \left( \mathcal{N}(\cdot| \As \xs_t + \Bs \us_t, \sigma^2 \Is_d) , \mathcal{N}(\cdot| \As' \xs_t + \Bs \us_t, \sigma^2 \Is_d)\right) \right] \\
		& = \frac{1}{2\sigma^2} \sum_{t=1}^T \mathbb{E}_{\mathbf{a}} \left[ \left\| \left(\As - \As'\right) \xs_t \right\|_2^2\right] = \epsilon^2 \mathbb{E}_{\mathbf{a}} \left[ \xs_{t,i}^2 \right].
	\end{align*}
	We proceed at properly bounding the KL-divergence, letting $\mathbf{e}_i$ be the $i$-th vector of the canonical basis of $\Reals^d$ and convening that $\xs_0 = \mathbf{0}_d$:
	\begin{align*}
	\mathbb{E}_{\mathbf{a}} \left[ \xs_{t,i}^2 \right] & =  \mathbb{E}_{\mathbf{a}} \left[ \left( \sum_{s=1}^{t-1} \mathbf{e}_i^\transpose\As^s \Bs \us_{t-s} + \sum_{s=1}^{t-1} \mathbf{e}_i^\transpose \As^s \epsilons_{t-s}\right)^2 \right] \\
	& =  \mathbb{E}_{\mathbf{a}} \left[ \left( (1-\rho) \sum_{s=1}^{t-1} \mathbf{a}_{i}^s \us_{t-s,i} + \sum_{s=1}^{t-1} \mathbf{a}_{i}^s \epsilons_{t-s,i}\right)^2 \right] \\
	& = \mathbb{E}_{\mathbf{a}} \left[  \underbrace{ (1-\rho)^2 \sum_{s=1}^{t-1} \sum_{l=1}^{t-1} \mathbf{a}_i^{s+l} \us_{t-s,i}  \us_{t-l,i}}_{\text{(a)}} +  2 \underbrace{(1-\rho)\sum_{s=1}^{t-1} \sum_{l=1}^{t-1}  \mathbf{a}_i^{s+l} \us_{t-s,i}   \epsilons_{t-l,i}}_{\text{(b)}} +  \underbrace{\sum_{s=1}^{t-1} \sum_{l=1}^{t-1} \mathbf{a}_i^{s+l} \epsilons_{t-s,i}  \epsilons_{t-l,i}}_{\text{(c)}} \right]
	\end{align*}
	Let us start with (a):
	\begin{align*}
	(1-\rho)^2 \mathbb{E}_{\mathbf{a}} \left[\sum_{s=1}^{t-1} \sum_{l=1}^{t-1} \mathbf{a}_i^{s+l} \us_{t-s,i}  \us_{t-l,i} \right] \le (1-\rho)^2 \sum_{s=1}^{t-1} \sum_{l=1}^{t-1} \rho^{s+l} \le 1,
	\end{align*}
	having observed that $|\us_{t-s,i} |, | \us_{t-l,i} | \le 1$, that $|\mathbf{a}_i| \le \rho$, and bounding the summations with the series. Let us move to (b):
	\begin{align*}
		(1-\rho)\mathbb{E}_{\mathbf{a}} \left[ \sum_{s=1}^{t-1} \sum_{l=1}^{t-1}  \mathbf{a}_i^{s+l} \us_{t-s,i}  \epsilons_{t-l,i} \right] & = (1-\rho)\mathbb{E}_{\mathbf{a}} \left[ \sum_{s=1}^{t-1} \sum_{l=s+1}^{t-1}  \mathbf{a}_i^{s+l} \us_{t-s,i}  \epsilons_{t-l,i} \right] \\
		& \quad + (1-\rho)\cancelto{0}{\mathbb{E}_{\mathbf{a}} \left[ \sum_{l=1}^{t-1} \sum_{s=l}^{t-1}  \mathbf{a}_i^{s+l} \us_{t-s,i}  \epsilons_{t-l,i} \right]} \\
		& \le (1-\rho) \sum_{s=1}^{t-1} \sum_{l=s+1}^{t-1}  \rho^{s+l} \mathbb{E}_{\mathbf{a}} \left[ |\epsilons_{t-l,i}| \right] \\
		& \le \frac{\sigma}{1-\rho}  \sqrt{\frac{2}{\pi}},
	\end{align*}
	having observed that $ \us_{t-s,i}$ and $\epsilons_{t-l,i}$ are independent when $s \ge l$ and  $\epsilons_{t-l,i}$ has zero mean, that $|\us_{t-s,i}| \le 1$, that $\mathbf{a}_i^{s+l} \le \rho^{s+l}$, and that the expectation of the absolute value of random variable normally distributed is given by $\E \left[ |\epsilons_{t-l,i}| \right] = \sigma \sqrt{\frac{2}{\pi}}$.
	Finally, let us consider (c):
	\begin{align*}
	\mathbb{E}_{\mathbf{a}} \left[ \sum_{s=1}^{t-1} \sum_{l=1}^{t-1} \mathbf{a}_i^{s+l} \epsilons_{t-s,i}  \epsilons_{t-l,i} \right] & = \mathbb{E}_{\mathbf{a}} \left[ \sum_{s=1}^{t-1} \mathbf{a}_i^{2s} \epsilons_{t-s,i}  \epsilons_{t-s,i} \right]  + 2 \cancelto{0}{\mathbb{E}_{\mathbf{a}} \left[ \sum_{s=1}^{t-2} \sum_{l=s+1}^{t-1} \mathbf{a}_i^{s+l} \epsilons_{t-s,i}  \epsilons_{t-l,i} \right]} \\
	& \le  \sigma^2 \sum_{s=1}^{t-1} \rho^{2s}  \le \frac{\sigma^2}{1-\rho^2} \le \frac{\sigma^2}{1-\rho},
	\end{align*}
	having observed that the noise vectors $\epsilons_{t-l,i} $ and $\epsilons_{t-s,i} $ are independent  whenever $s \neq l$, that $\mathbb{E}_{\mathbf{a}}[\epsilons_{t-s,i}^2]=\sigma^2$, and having bounded the sum with the series. Coming back to the original bound, we have:
	\begin{align*}
	\mathbb{E}_{\mathbf{a}} \left[ \xs_{t,i}^2 \right]  \le 1 + \frac{1}{1-\rho} \left( \sigma^2 + 2\sigma \sqrt{\frac{2}{\pi}}\right).
	\end{align*}
	
	For $i \in \dsb{d}$ and $\mathbf{a} \in \mathcal{A}$, we introduce the symbol:
	\begin{align*}
		p_{\mathbf{a},i} = \mathbb{P}_{\mathbf{a}} \left( \sum_{t=1}^T \indic \{ \mathrm{sign}(\us_{t,i}) \neq \mathrm{sign}(\hs_{\mathbf{a},i}) \} \ge \frac{T}{2} \right).
	\end{align*}
	Thus, for $\mathbf{a}$ and $\mathbf{a}'$ defined as above, by the Bretagnolle-Huber inequality~\citep[][Theorem 14.2]{lattimore2020bandit}, we have:
	\begin{align*}
		p_{\mathbf{a},i} + p_{\mathbf{a}',i} & \ge \frac{1}{2} \exp \left(- D_{\text{KL}} \left(\mathbb{P}_{\mathbf{a}} , \mathbb{P}_{\mathbf{a}'} \right) \right) = \frac{1}{2} \exp \left(-\frac{1}{2\sigma^2} \sum_{t=1}^T \mathbb{E}_{\mathbb{P}} \left[ \left\| \left(\As - \As'\right) \xs_t \right\|_2^2\right] \right)\\
		& \ge \frac{1}{2} \exp \left( -\frac{T\epsilon^2}{2} \left( \frac{1}{\sigma^2} + \frac{1}{1-\rho} \left(1 + \frac{2}{\sigma} \sqrt{\frac{2}{\pi}} \right) \right) \right) \\
		& \ge \frac{1}{2} \exp \left( -\frac{2 T\epsilon^2}{1-\rho}  \right),
	\end{align*}
	having selected $\sigma^2=1$.
	We use the notation $\sum_{\mathbf{a}_{-i}}$ to denote the multiple summation $ \sum_{\mathbf{a}_1,\dots,\mathbf{a}_{i-1},\mathbf{a}_{i+1},\dots,\mathbf{a}_d \in \{\rho,\rho-\epsilon\}^{d-1}}$:
	\begin{align*}
		\sum_{\mathbf{a} \in \mathcal{A}} 2^{-d} \sum_{i=1}^d p_{\mathbf{a},i} & = \sum_{i=1}^d \sum_{\mathbf{a}_{-i}} 2^{-d} \sum_{\mathbf{a}_i \in \{\rho,\rho-\epsilon\}} p_{\mathbf{a},i} \\
		& \ge \sum_{i=1}^d \sum_{\mathbf{a}_{-i}} 2^{-d} \cdot \frac{1}{2} \exp\left( - \frac{2 T \epsilon^2}{1-\rho} \right) \\
		& = \frac{d}{4}\exp\left( - \frac{2 T \epsilon^2}{1-\rho} \right).
	\end{align*}
	Therefore, with this averaging argument, we can conclude that there exists $\mathbf{a}^* \in \mathcal{A}$ such that $\sum_{i=1}^d p_{\mathbf{a}^*,i} \ge \frac{d}{4}\exp\left( - \frac{2 T \epsilon^2}{1-\rho} \right)$. For this choice $\mathbf{a}^*$, we consider $\us^*_{\mathbf{a}^*} = \mathrm{sign}(\mathbf{a}^*) \in \Us$, we can proceed to the lower bound on the expected offline regret:
	\begin{align*}
		\mathbb{E} R^{\text{off}}(\bm{\underline{\pi}}, T) & = \sum_{t=1}^T \mathbb{E}_{\mathbf{a}^*}\left[ \inner{\hs_{\mathbf{a}^*}}{\us^*_{\mathbf{a}^*} - \us_t} \right] \\
		& = \sum_{t=1}^T \mathbb{E}_{\mathbf{a}^*} \left[ \sum_{i=1}^d \indic\{ \mathrm{sign}(\us_{t,i}) \neq \mathrm{sign}({\hs_{\mathbf{a}^*,i}})\}  \frac{\epsilon}{1-(\rho-\epsilon)} \right] \\
		& = \frac{\epsilon}{1-(\rho-\epsilon)} \sum_{t=1}^T \sum_{i=1}^d \mathbb{P}_{\mathbf{a}^*} \left(  \mathrm{sign}(\us_{t,i}) \neq \mathrm{sign}({\hs_{\mathbf{a}^*,i}}) \right) \\
		& \ge \frac{T \epsilon }{2 (1-(\rho-\epsilon))} \sum_{i=1}^d \mathbb{P}_{\mathbf{a}^*} \left( \sum_{t=1}^T \indic\{ \mathrm{sign}(\us_{t,i}) \neq \mathrm{sign}({\hs_{\mathbf{a}^*,i}})\} \ge \frac{T}{2} \right) \\
		& = \frac{T \epsilon }{2 (1-(\rho-\epsilon))} \sum_{i=1}^d p_{\mathbf{a}^*,i} \ge \frac{T d \epsilon }{8 (1-(\rho-\epsilon))} \exp\left( - \frac{2 T \epsilon^2}{1-\rho} \right).
	\end{align*}
	We now maximize over $0 \le \epsilon < \rho$. To this end, we perform the substitution $\epsilon = \frac{(1-\rho)\widetilde{\epsilon}}{1-\widetilde{\epsilon}}$, with $0 \le \widetilde{\epsilon} \le {\rho}$:
	\begin{align*}
	\frac{T d \epsilon }{8 (1-(\rho-\epsilon))} \exp\left( - \frac{2 T \epsilon^2}{1-\rho} \right) = \frac{T d\widetilde{\epsilon} }{8} \exp \left( -\frac{2\widetilde{\epsilon}^2 T (1-\rho)}{( 1- \widetilde{\epsilon})^2} \right) \ge \frac{T d\widetilde{\epsilon} }{8} \exp \left( -8\widetilde{\epsilon}^2 T (1-\rho) \right),
	\end{align*}
	where the last inequality holds for $\widetilde{\epsilon} \le \frac{1}{2}$. We not take $\widetilde{\epsilon} = \frac{1}{\sqrt{8T(1-\rho)}}$ which is smaller than $\frac{1}{2}$ if $T \ge \frac{1}{2(1-\rho)}$, to get:
	\begin{align*}
	\mathbb{E} R^{\text{off}}(\bm{\underline{\pi}}, T) \ge \frac{d \sqrt{T}}{\sqrt{512e(1-\rho)}}.
	\end{align*}
	Notice that with this choice of $\widetilde{\epsilon}$ (and, consequently, of $\epsilon$), for sufficiently large $T$, we fulfill Assumption~\ref{ass:boundedness}. Indeed:
	\begin{align*}
		\vtheta = -1 + \frac{1}{\sqrt{32 T(1-\rho)}}, \quad J^*_{\mathbf{a}} = \frac{d}{\sqrt{32T(1-\rho)}}.
	\end{align*}
	Thus, we require $T \ge \mathcal{O} \left( \frac{d^2}{1-\overline{\rho}} \right)$. Finally, to convert this result to the expected regret, we employ Lemma~\ref{lemma:offOnRel}:
	\begin{align*}
	\mathbb{E} R^{\text{off}}(\bm{\underline{\pi}}, T) \ge \mathbb{E} R^{\text{off}}(\bm{\underline{\pi}}, T) - \frac{d}{1-\rho}.
	\end{align*}
	Under the constraint $T \ge \mathcal{O} \left( \frac{d^2}{1-\overline{\rho}} \right)$, we observe that:
	\begin{align*}
	\mathbb{E} R^{\text{off}}(\bm{\underline{\pi}}, T)  \ge \Omega \left( \frac{d \sqrt{T}}{(1-\rho)^{\frac{1}{2}}} \right).
	\end{align*}
\end{proof}

\optimalPolicy*
\begin{proof}
	Referring to the notation of Appendix~\ref{apx:finiteHorizon}, we first observe that for every policy $\underline{\vpi}$, we have $J(\underline{\vpi}) = \liminf_{H \rightarrow + \infty} J_H(\underline{\vpi})$, where $J_H(\underline{\vpi}) = \frac{1}{H} \E[\sum_{t=1}^H y_t]$, is the $H$-horizon expected average reward. Let us start with Equation~\eqref{p:002}, a fixed finite $H \in \Nat$, and considering the sequence of actions $(\us_1, \us_2, \dots)$ generated by policy $\underline{\vpi}$:
	\begin{align*}
	    J_{H}(\underline{\vpi}) & = \frac{1}{H} \sum_{s=1}^H \inner{\hs^{\dsb{0,H-s}}}{ \E[\us_{s}]} +  \frac{1}{H} \sum_{t=1}^H \vomega^\transpose \As^{t-1} \E[\xs_1] \\
	    & =  \frac{1}{H} \sum_{s=1}^H \inner{\hs}{ \E[\us_{s}]} - \frac{1}{H} \sum_{s=1}^H \inner{\hs^{\srb{H-s+1,+\infty}}}{\E[\us_{s}]} + \frac{1}{H} \sum_{t=1}^H \vomega^\transpose \As^{t-1} \E[\xs_1] .
	\end{align*}
	Now, we consider two bounds on $J_{H}(\underline{\vpi})$, obtained by an application of Cauchy-Schwarz inequality on the second addendum:
	\begin{align*}
	    & J_{H}(\underline{\vpi}) \le  \frac{1}{H} \sum_{s=1}^H \inner{\hs}{ \E[\us_{s}]} + \frac{1}{H} \sum_{s=1}^H \left\|\hs^{\srb{H-s+1,+\infty}}\right\|_2 \|\E[\us_{s}]\|_2 \\
	    & \qquad\qquad\qquad  \qquad\qquad\qquad + \frac{1}{H} \sum_{t=1}^H \vomega^\transpose \As^{t-1} \E[\xs_1] \eqqcolon J_{H}^{\uparrow}(\underline{\vpi}) , \\
	    & J_{H}(\underline{\vpi}) \ge  \frac{1}{H} \sum_{s=1}^H \inner{\hs}{ \E[\us_{s}]} - \frac{1}{H} \sum_{s=1}^H \left\|\hs^{\srb{H-s+1,+\infty}}\right\|_2 \|\E[\us_{s}]\|_2 \\
	    & \qquad\qquad\qquad \qquad\qquad\qquad + \frac{1}{H} \sum_{t=1}^H \vomega^\transpose \As^{t-1} \E[\xs_1] \eqqcolon J_{H}^{\downarrow}(\underline{\vpi}).
	\end{align*}
	Concerning the term $\|\E[\us_{s}]\|_2$, we have that $\|\E[\us_{s}]\|_2 \le \E[ \|\us_s\|_2] \le U$, having used Jensen's inequality and under Assumption~\ref{ass:boundedness}. Regarding the second term, using Assumptions~\ref{ass:spectralNorm} and~\ref{ass:boundedness}, we obtain:
	\begin{align}
	    \left\|\hs^{\srb{H-s+1,+\infty}}\right\|_2 & = \left\| \sum_{l=H-s+1}^{+\infty} \Bs^\transpose (\As^{l-1})^\transpose \vomega \right\|_2 \notag\\
	    & \le B \Omega \sum_{l=H-s+1}^{+\infty} \Phi(\As) \rho(\As)^{l-1} \notag\\
	    & =  B \Omega \Phi(\As) \frac{\rho(\As)^{H-s}}{1-\rho(\As)}. \label{p:-100}
	\end{align}
	Plugging this result into the summation over $s$, we obtain:
	\begin{align*}
	   \frac{1}{H} \cdot \frac{B \Omega \Phi(\As)}{1-\rho(\As)} \sum_{s=1}^H  \rho(\As)^{H-s} = \frac{B \Omega \Phi(\As)(1-\rho(\As)^H)}{H(1-\rho(\As))^2}.
	\end{align*}
	It is simple to observe that the last term approaches zero as $H\rightarrow +\infty$. Moreover, with an analogous argument, it can be proved that $\left\| \frac{1}{H} \sum_{t=1}^H \vomega^\transpose \As^{t-1} \E[\xs_1] \right\|_2 \rightarrow 0$ as $H\rightarrow +\infty$. Thus, we have that $\liminf_{H \rightarrow +\infty} J_{H}^{\downarrow}(\underline{\vpi}) = \liminf_{H \rightarrow +\infty} J_{H}^{\uparrow}(\underline{\vpi})$. Consequently, by the squeezing theorem of limits, we have:
	\begin{align*}
	    J(\underline{\vpi}) & = \liminf_{H \rightarrow +\infty} J_{H}^{\uparrow}(\underline{\vpi}) =  \liminf_{H \rightarrow +\infty} J_{H}^{\downarrow}(\underline{\vpi}) \\
	    & = \liminf_{H \rightarrow +\infty} \frac{1}{H} \sum_{s=1}^H \inner{\hs}{ \E[\us_{s}]}  = \hs^\transpose \left(\liminf_{H \rightarrow +\infty} \frac{1}{H} \sum_{s=1}^H  \E[\us_{s}]\right).
	\end{align*}
	It follows that an optimal policy is a policy that plays the constant action $\us^* \in \argmax_{\us \in \Us} \inner{\hs}{\us}$.
\end{proof}

\subsection{Proofs of Section~\ref{sec:algorithm}}\label{apx:Proofalgorithm}

\concentration*

\begin{proof}
First of all, let us properly relate the round $t \in \dsb{T}$ and the index of the epoch $m \in \dsb{M}$. For every epoch $m \in \dsb{M}$, we denote with $t_m$ the last round of epoch $m$ (\ie the one in which we update the relevant matrices $\Vs_t$ and $\bs_t$):\footnote{It is worth noting that the variables $t_m$ are deterministic.}
\begin{align*}
	t_0 = 0, \qquad t_{m} = t_{m-1} + 1 + H_m.
\end{align*}
We now proceed to define suitable filtrations. Let $\mathbb{F} = (\mathcal{F}_t)_{t \in \dsb{T}}$ such that for every $t \ge 1$, the random variables $\{\us_1, y_1,\dots,\us_{t-1}, y_{t-1}, \us_t\}$ are $\mathcal{F}_{t-1}$-measurable, \ie $\mathcal{F}_{t-1} = \sigma(\us_1, y_1,\dots,\us_{t-1}, y_{t-1}, \us_t)$. Let us also consider the filtration indexed by $m$, denoted with $\widetilde{\mathbb{F}} = (\widetilde{\mathcal{F}}_m)_{m \in \dsb{M}}$ and defined for all $m \in \dsb{M}$ as  $\widetilde{\mathcal{F}}_m = \mathcal{F}_{t_{m+1}-1}$. Thus, the random variables $\widetilde{\mathcal{F}}_{m-1}$-measurable are those realized until the end of epoch $m$ except for $y_{t_m}$.

Since the estimates $\widehat{\hs}_t$ do not change within an epoch, we need to guarantee the statement for all rounds $\{t_{m}\}_{m \in \dsb{M}}$ only. For these rounds, we define the following quantities:
\begin{align*}
	& \widetilde{y}_m = y_{t_m}, \\
	& \widetilde{\us}_m = \us_{t_m}, \qquad \text{(or any $\us_l$ with $l \in \dsb{t_{m-1}+1,t_m}$ since they are all equal)}\\
	& \widetilde{\xi}_m = \eta_{t_m} + \sum_{s=1}^{H_m+1} \vomega^\transpose \As^{s-1} \epsilons_{t_m-s}, \\
	& \widetilde{\xs}_{m-1} = \xs_{t_{m-1}}, \\
	& \widetilde{\hs}_m = \widehat{\hs}_{t_m},\\
	& \widetilde{\Vs}_{m} = \Vs_{t_m}, \\
	& \widetilde{\bs}_{m} = \bs_{t_m}.
\end{align*}
We prove that $(\widetilde{\xi}_m )_{m \in \dsb{M}}$ is a martingale difference process adapted to the filtration $\widetilde{\mathbb{F}} $. To this end, we recall that, by construction, $(\eta_t)_{t \in \dsb{T}}$ and $(\epsilons_t)_{t \in \dsb{T}}$ are martingale difference processes adapted to the filtration $\mathbb{F}$. It is clear that $\widetilde{\xi}_m$ is $\mathcal{F}_m$-measurable and, being $\sigma^2$-subgaussian it is absolutely integrable. Furthermore, using the tower law of expectation:
\begin{align*}
	\E\left[ \widetilde{\xi}_m | \widetilde{\mathcal{F}}_{m-1}\right] & =  \E\left[ \eta_{t_m} + \sum_{s=1}^{H_m+1} \vomega^\transpose \As^{s-1} \epsilons_{t_m-s} | \mathcal{F}_{t_{m}-1}\right]  \\
	& = \E\left[ \eta_{t_m} | \mathcal{F}_{t_{m}-1} \right] + \E\left[\sum_{s=1}^{H_m+1} \vomega^\transpose \As^{s-1}  \E[\epsilons_{t_m-s} |\mathcal{F}_{t_{m}-s-1}] | \mathcal{F}_{t_{m}-1} \right]=0,
\end{align*}
since the system is operating by persisting the action after having decided it at the beginning of the epoch. Thus, by exploiting the decomposition in Equation~\eqref{eq:decompY}, we can write:
\begin{align}
	\widetilde{y}_m  = y_{t_m} & = \inner{\hs^{\dsb{0,H_m+1}}}{ \widetilde{\us}_m} + \vomega^\transpose \As^{H_m+1} \xs_{t_{m-1}} + \eta_{t_m}  + \sum_{s=1}^{H_m+1} \vomega^\transpose \As^{s-1} \epsilons_{t_m-s} \notag\\
	& = \inner{\hs^{\dsb{0,H_m+1}}}{\widetilde{\us}_m} + \vomega^\transpose \As^{H_m+1} \widetilde{\xs}_{m-1} + \widetilde{\xi}_m \notag\\
	& =  \inner{\hs}{ \widetilde{\us}_m} - \inner{\hs^{\srb{H_m+2,\infty}}}{ \widetilde{\us}_m } + \vomega^\transpose \As^{H_m+1} \widetilde{\xs}_{m-1} + \widetilde{\xi}_m,\label{p:-1000}
\end{align}
where we simply exploit the identity $\hs = \hs^{\dsb{0,H_m+1}} + \hs^{\srb{H_m+2,\infty}} $. We now introduce the following vectors and matrices:
\begin{align*}
	&\widetilde{\mathbf{U}}_m = 
		\begin{pmatrix}
			\widetilde{\us}_{1}^\transpose \\
			\vdots\\
			\widetilde{\us}_{m}^\transpose
		\end{pmatrix} \in \Reals^{m \times d}, \qquad
		&&\widetilde{\mathbf{y}}_m = 
		\begin{pmatrix}
			\widetilde{y}_{1} \\
			\vdots\\
			\widetilde{y}_{m} 
		\end{pmatrix} \in \Reals^{m},
\\
& \widetilde{\bm{\xi}}_m = \begin{pmatrix}
			\widetilde{\xi}_{1} \\
			\vdots\\
			\widetilde{\xi}_{m} 
		\end{pmatrix} \in \Reals^{m}, \qquad 
&& \widetilde{\bm{\nu}}_m = \begin{pmatrix}
			\vomega^\transpose \As^{H_1+2} \widetilde{\xs}_0 \\
			\vdots\\
			\vomega^\transpose \As^{H_m+2} \widetilde{\xs}_{m-1}
		\end{pmatrix} \in \Reals^{m},\\
		& \widetilde{\mathbf{g}}_m = \begin{pmatrix}
			\inner{\hs^{\srb{H_1+1,\infty}}}{ \widetilde{\us}_1} \\
			\vdots\\
			\inner{\hs^{\srb{H_m+1,\infty}}}{ \widetilde{\us}_m}
		\end{pmatrix} \in \Reals^{m}.
\end{align*}
Using the vectors and matrices above, we observe that $\widetilde{\Vs}_m = \lambda \Is + \widetilde{\mathbf{U}}_m^\transpose \widetilde{\mathbf{U}}_m$ and $\widetilde{\bs}_m = \widetilde{\mathbf{U}}_m^\transpose \widetilde{\ys}_m$. Furthermore, by exploiting Equation~\eqref{p:-1000}, we can write:
\begin{align*}
	\widetilde{\ys}_m = \widetilde{\mathbf{U}}_m \hs - \widetilde{\mathbf{g}}_m + \widetilde{\bm{\nu}}_m + \widetilde{\bm{\xi}}_m.
\end{align*}
Let us consider the estimate at $m \in \dsb{M}$:
\begin{align*}
	\widetilde{\hs}_m =  \widetilde{\Vs}_{m}^{-1}\widetilde{\bs}_{m} & = \left( \lambda \Is + \widetilde{\mathbf{U}}_m^\transpose \widetilde{\mathbf{U}}_m \right)^{-1} \widetilde{\mathbf{U}}_m^\transpose \widetilde{\ys}_m \\
	& = \left( \lambda \Is + \widetilde{\mathbf{U}}_m^\transpose \widetilde{\mathbf{U}}_m \right)^{-1} \widetilde{\mathbf{U}}_m^\transpose  \left(\widetilde{\mathbf{U}}_m \hs- \widetilde{\mathbf{g}}_m + \widetilde{\bm{\nu}}_m + \widetilde{\bm{\xi}}_m \right)\\
	& =\hs + \left( \lambda \Is + \widetilde{\mathbf{U}}_m^\transpose \widetilde{\mathbf{U}}_m \right)^{-1}  \left( - \lambda  \hs - \widetilde{\mathbf{U}}_m^\transpose \widetilde{\mathbf{g}}_m +\widetilde{\mathbf{U}}_m^\transpose \widetilde{\bm{\nu}}_m + \widetilde{\mathbf{U}}_m^\transpose \widetilde{\bm{\xi}}_m \right).
\end{align*}
We now proceed at bounding the $\|\cdot\|_{\widetilde{\Vs}_m}$-norm, and exploit the triangle inequality:
\begin{align*}
 \left\| \widetilde{\hs}_m - \hs\right\|_{\widetilde{\Vs}_m} &\le \lambda \left\| \widetilde{\Vs}_m^{-1} \hs\right\|_{\widetilde{\Vs}_m} + \left\| \widetilde{\Vs}_m^{-1} \widetilde{\mathbf{U}}_m^\transpose \widetilde{\mathbf{g}}_m \right\|_{\widetilde{\Vs}_m} + \left\| \widetilde{\Vs}_m^{-1} \widetilde{\mathbf{U}}_m^\transpose \widetilde{\bm{\nu}}_m \right\|_{\widetilde{\Vs}_m} + \left\| \widetilde{\Vs}_m^{-1} \widetilde{\mathbf{U}}_m^\transpose \widetilde{\bm{\xi}}_m  \right\|_{\widetilde{\Vs}_m} \\
 & = \underbrace{\lambda \left\|  \hs \right\|_{\widetilde{\Vs}_m^{-1}}}_{\text{(a)}} + \underbrace{\left\|  \widetilde{\mathbf{U}}_m^\transpose \widetilde{\mathbf{g}}_m \right\|_{\widetilde{\Vs}_m^{-1}}}_{\text{(b)}} + \underbrace{\left\| \widetilde{\mathbf{U}}_m^\transpose \widetilde{\bm{\nu}}_m \right\|_{\widetilde{\Vs}_m^{-1}}}_{\text{(c)}} + \underbrace{\left\|\widetilde{\mathbf{U}}_m^\transpose \widetilde{\bm{\xi}}_m  \right\|_{\widetilde{\Vs}_m^{-1}}}_{\text{(d)}}, 
\end{align*}
where we simply exploited the identity $\|\Vs^{-1} \xs \|_{\Vs}^2 = \xs^\transpose \Vs^{-1} \Vs \Vs^{-1} \xs = \xs^\transpose \Vs^{-1}  \xs = \| \xs \|_{\Vs^{-1}}^2$. We now bound one term at a time. Let us start with (a):
\begin{align*}
	\text{(a)}^2 = \lambda^2 \left\|  \hs\right\|_{\widetilde{\Vs}_m^{-1}}^2 & = \lambda^2   \hs^\transpose \widetilde{\Vs}_m^{-1}\hs \\
	& \le \lambda^2 \left\| \widetilde{\Vs}_m^{-1} \right\|_2  \left\| \hs \right\|_2^2 \\
	& \le \lambda \left\| \hs \right\|_2^2 \\
	& \le \lambda \left(  \Theta + \frac{\Omega B \Phi(\As) }{1-\rho(\As)} \right)^2,
\end{align*}
where we observed that $ \left\| \widetilde{\Vs}_m^{-1} \right\|_2 \le  \left\| \widetilde{\Vs}_m \right\|_2^{-1} \le \lambda^{-1}$. Finally, we have bounded the norm of $\hs$:
\begin{align*}
\left\|\hs\right\|_2 & = \left\| \sum_{s =0}^{+\infty} \hs^{\{s\}} \right\|_2\\
& \le  \sum_{s =0}^{+\infty} \left\| \hs^{\{s\}} \right\|_2 \\
& \le \| \vtheta \|_2 + \|\vomega\|_2 \|\Bs\|_2 \sum_{s =1}^{+\infty} \| \As \|^{s-1} \\
& \le \Theta + \frac{\Omega B \Phi(\As) }{1-\rho(\As)},
\end{align*}
where we have exploited Assumptions~\ref{ass:spectralNorm} and~\ref{ass:boundedness}.

We now move to term (b):
\begin{align*}
	\text{(b)}^2=\left\|  \widetilde{\mathbf{U}}_m^\transpose \widetilde{\mathbf{g}}_m \right\|_{\widetilde{\Vs}_m^{-1}}^2 & = \widetilde{\mathbf{g}}_m^\transpose \widetilde{\mathbf{U}}_m \widetilde{\Vs}_m^{-1} \widetilde{\mathbf{U}}_m^\transpose \widetilde{\mathbf{g}}_m \\
	& \le \frac{1}{\lambda} \left\|\widetilde{\mathbf{g}}_m^\transpose \widetilde{\mathbf{U}}_m  \right\|_2^2 \\
	& = \frac{1}{\lambda} \left\|\sum_{l=1}^m \inner{ \widetilde{\us}_l}{ \hs^{\srb{H_l+2,\infty}}} \widetilde{\us}_l  \right\|_2^2 \\
	& \le \frac{1}{\lambda} \left(\sum_{l=1}^m \| \widetilde{\us}_l\|_2^2 \left\| \hs^{\srb{H_l+2,\infty}} \right\|_2 \right)^2\\
	& \le \frac{U^4 \Omega^2 B^2 \Phi(\As)^2}{\lambda(1-\rho(\As))^2} \cdot \left( \sum_{l=1}^m \rho(\As)^{H_l+1} \right)^2,
\end{align*}
where we have employed the following inequality:
\begin{align*}
	\left\|\hs^{\srb{H_l+2,\infty}}\right\|_2 & = \left\| \vomega^\transpose  \sum_{j=H_l+2}^{+\infty} \As^{j-1} \Bs\right\|_2 \\
	& \le \left\| \vomega \right\|_2 \left\| \Bs\right\|_2 \sum_{j=H_l+2}^{+\infty} \left\|  \As^{j-1}\right\|_2  \\
	& \le \Omega B \Phi(\As) \frac{\rho(\As)^{H_l+1}}{1-\rho(\As)}.
	\end{align*}
Let us now consider term (c):
\begin{align*}
	\text{(c)}^2 = \left\| \widetilde{\mathbf{U}}_m^\transpose \widetilde{\bm{\nu}}_m \right\|_{\widetilde{\Vs}_m^{-1}}^2 & = \widetilde{\bm{\nu}}_m^\transpose \widetilde{\mathbf{U}}_m\widetilde{\Vs}_m^{-1}\widetilde{\mathbf{U}}_m^\transpose \widetilde{\bm{\nu}}_m \\
	& \le \frac{1}{\lambda} \left\| \widetilde{\mathbf{U}}_m^\transpose \widetilde{\bm{\nu}}_m \right\|^2_{2}\\
	& = \frac{1}{\lambda} \left\| \sum_{s=1}^{m} \vomega^\transpose \As^{H_l+1} \widetilde{\xs}_{l-1} \widetilde{\us}_l \right\|_2^2 \\
	& \le \frac{1}{\lambda} \left( \sum_{s=1}^{m} \| \vomega\|_2 \left\|\As^{H_l+1} \right\|_2 \|\widetilde{\xs}_{l-1}\|_2 \| \widetilde{\us}_l\|_2 \right)^2  \\
	& \le \frac{X^2 \Omega^2 U^2 \Phi(\As)^2}{\lambda} \cdot \left( \sum_{l=1}^m \rho(\As)^{H_l+1} \right)^2.
\end{align*}
We now bound the summations, exploiting the inequality $\rho(\As) \le \overline{\rho}$, holding by assumption:
\begin{align*}
	\sum_{l=1}^m \rho(\As)^{H_l+1} & = \sum_{l=1}^m \rho(\As)^{\Big\lfloor\frac{\log l}{\log \frac{1}{\overline{\rho}}} \Big\rfloor+1} \\
	& \le \sum_{l=1}^m \rho(\As)^{\frac{\log l}{\log \frac{1}{\overline{\rho}}}} \\
	& = \sum_{l=1}^m \exp\left( -\frac{\log \frac{1}{\rho(\As)}}{\log \frac{1}{\overline{\rho}}}{\log l}  \right) \\
	& = \sum_{l=1}^m \frac{1}{l} \le \log(m+1) +1  \le \log(t+1)+1 = \log(e(t+1)),
\end{align*}
having exploited the fact that $m \le t$ and the bound with the integral to the harmonic sum.

Finally, we consider term (d). In this case, we apply Theorem 1 of~\citep{abbasi2011improved}, observing that the conditions are satisfied. To this end, we first need to determine the subgaussianity constant for the noise process $\widetilde{\xi}_l$. For every $l \in \dsb{m}$ and $\zeta \in \Reals$, and properly using the tower law of expectation:
\begin{align*}
		\E \left[ \exp \left( \zeta  \widetilde{\xi}_{l} \right) | \widetilde{\mathcal{F}}_{l-1} \right] & =  \E \left[ \exp \left( \zeta \eta_{t_l} + \zeta\sum_{s=1}^{H_m+1}\vomega^\transpose \As^{s-1} \epsilons_{t_l-s} \right) |\mathcal{F}_{t_l-1}  \right] \\
		& = \E \left[ \exp \left( \zeta \eta_{t_l} \right)| \mathcal{F}_{t_l-1}\right] \prod_{s=1}^{H_m+1}\E \left[ \E \left[ \exp \left( \zeta \vomega^\transpose \As^{s-1} \epsilons_{t_l-s} \right) | \mathcal{F}_{t_l-1-s}\right]| \mathcal{F}_{t_l-1} \right]\\
		& \le \exp \left( \frac{\zeta^2\sigma^2}{2} \right) \prod_{s=1}^{H_m+1}\E \left[  \exp \left( \frac{\zeta^2 \|\vomega^\transpose \As^{s-1}\|_2^2 \sigma^2}{2} \right) | \mathcal{F}_{t_l-1} \right]\\
		& \le \exp \left( \frac{\zeta^2\sigma^2}{2} \right) \prod_{s=1}^{H_m+1}  \exp \left( \frac{\zeta^2 \Omega^2 \Phi(\As)^2 \rho(\As)^{2(s-1)} \sigma^2}{2} \right)\\
		&\le  \exp \left( \frac{\sigma^2\zeta^2}{2} \left( 1 + \Omega^2 \Phi(\As)^2 \sum_{s=1}^{+\infty} \rho(\As)^{2(s-1)} \right) \right) \\
		& =  \exp \left( \frac{\sigma^2\zeta^2}{2} \left( 1 + \frac{\Omega^2 \Phi(\As)^2}{(1-\rho(\As)^2)}\right) \right).
	\end{align*}

Thus, simultaneously for all $m \in \dsb{M}$, with probability at least $1-\delta$, it holds that:
\begin{align*}
\text{(d)}^2 = \left\|\widetilde{\mathbf{U}}_m^\transpose \widetilde{\bm{\xi}}_m  \right\|_{\widetilde{\Vs}_m^{-1}}^2 \le 2  \sigma^2 \left( 1 + \frac{\Omega^2 \Phi(\As)^2}{(1-\rho(\As)^2)}\right)  \left( \log \left(\frac{1}{\delta} \right) + \frac{1}{2} \log \left( \frac{\det\left(\widetilde{\Vs}_m\right)}{  \lambda^d} \right) \right).
\end{align*}
\end{proof}

We now proceed at bounding the offline regret $R^{\text{off}}$ and, then, relating the offline regret $R^{\text{off}}$ with the online regret $R$, as defined in the main paper.

\begin{restatable}[Offline Regret Upper Bound]{thr}{regretOff}\label{thr:regretOff}
    Under Assumptions~\ref{ass:spectralNorm} and~\ref{ass:boundedness}, having selected $\beta_t$ as in Equation~\eqref{eq:beta}, for every $\delta \in (0,1)$, with probability at least $1-\delta$, \algnameshort suffers an offline regret $R^{\text{off}}$ bounded as:
    \begin{align*}
        R^{\text{off}}(\bm{\underline{\pi}}^{\text{\algnameshort}},T) \le\sqrt{8d T  {\beta}_{T-1}^2 \left(1 + \frac{\log T}{\log \frac{1}{\overline{\rho}}} \right) \log \left( 1+ \frac{T U^2}{d\lambda} \right)}.
    \end{align*}
    Moreover, by setting $\delta=1/T$, highlighting the dependencies on $T$, $\overline{\rho}$, $d$, and $\sigma$ only, the expected offline regret $\E R^{\text{off}}$ is bounded as:
    \begin{align*}
        \E R^{\text{off}}(\bm{\underline{\pi}}^{\text{\algnameshort}}, T) \le \mathcal{O} \left( \frac{d \sigma \sqrt{T} (\log T)^{\frac{3}{2}}}{1-\overline{\rho}} + \frac{\sqrt{d T} (\log T)^2}{(1-\overline{\rho})^{\frac{3}{2}}} \right).
    \end{align*}
\end{restatable}

\begin{proof}
For every epoch $m \in \dsb{M}$, let us define $\widetilde{\beta}_{m-1} = \beta_{t_{m-1}}$ and define the confidence set $\mathcal{C}_{m-1} = \{ \widetilde{\hs} \in \Reals^d : \| \widetilde{\hs} - \widetilde{\hs}_{m-1} \|_{\widetilde{\Vs}_{m-1}} \le \widetilde{\beta}_{m-1} \}$.
Let us start by considering the instantaneous offline regret $\widetilde{r}_m$ at epoch $m \in \dsb{M}$. Let $\us^* \in \argmax_{\us \in \Us}\, \inner{\hs}{\us}$ and let $\widetilde{\hs}_{m-1}^{\uparrow} \in \mathcal{C}_{m-1}$ such that $\text{UCB}_{t_{m-1}+1}(\widetilde{\us}_m) = \inner{\widetilde{\hs}_{m-1}^{\uparrow}}{\widetilde{\us}_m}$. Thus, with probability at least $1-\delta$, we have:
\begin{align}
		\widetilde{r}_m & = J^* - J(\widetilde{\us}_m) = \inner{\hs}{\us^*} - \inner{\hs}{\widetilde{\us}_m} \pm \inner{\widetilde{\hs}_{m-1}^\uparrow}{\widetilde{\us}_m} \notag\\
		& \le \inner{\widetilde{\hs}_{m-1}^{\uparrow} - \hs}{\widetilde{\us}_m}  \label{l:001}\\
		& \le \left\|\widetilde{\hs}_{m-1}^{\uparrow} - \hs \right\|_{\widetilde{\Vs}_{m-1}} \left\|\widetilde{\us}_m\right\|_{\widetilde{\Vs}_{m-1}^{-1}} \notag \\
		& \le \left( \left\|\widetilde{\hs}_{m-1}^{\uparrow} - \widetilde{\hs}_{m-1} \right\|_{\widetilde{\Vs}_{m-1}} + \left\|\widetilde{\hs}_{m-1} - \hs \right\|_{\widetilde{\Vs}_{m-1}}\right) \label{l:002} \left\|\widetilde{\us}_m\right\|_{\widetilde{\Vs}_{m-1}^{-1}} \\
		& \le 2 \widetilde{\beta}_{m-1} \left\|\widetilde{\us}_m\right\|_{\widetilde{\Vs}_{m-1}^{-1}}.\label{l:003}
	\end{align}
	where line~\eqref{l:001} follows from the optimism, line~\eqref{l:002} derives from triangle inequality, line~\eqref{l:003} is obtained by observing that $\hs \in \mathcal{C}_{m-1}$ with probability at least $1-\delta$, simultaneously for all $m \in \dsb{M}$, thanks to Theorem~\ref{thr:concentration}, having observed that $\widetilde{\beta}_{m-1}$ is larger than the right hand side of Theorem~\ref{thr:concentration}.
	
	We now move to the cumulative offline regret over the whole horizon $T$, by decomposing w.r.t. the epochs and recalling that we pay the same instantaneous regret within each epoch:
	\begin{align*}
		R^{\text{off}}(\text{\algnameshort}, T) = \sum_{m=1}^M (H_m + 1) \widetilde{r}_m  \le \sqrt{\sum_{m=1}^M (H_m+1)^2} \sqrt{ \sum_{m=1}^M \widetilde{r}_m^2 }.
	\end{align*}
	Concerning the first summation, we proceed as follows, recalling that $M \le T$ and $H_m \le H_M$ for all $m \in \dsb{M}$:
	\begin{align*}
	    \sum_{m=1}^M (H_m+1)^2 \le T (H_M+1) \le T \left(1 + \frac{\log T}{\log \frac{1}{\overline{\rho}}} \right).
	\end{align*}
	For the second summation, we follow the usual derivation for linear bandits, recalling that $ \widetilde{\beta}_{M-1} \ge \max\{1, \widetilde{\beta}_{m-1} \}$ for all $m \in \dsb{M}$ and that under Assumption~\ref{ass:boundedness} we have that $\widetilde{r}_m^2 \le 2$. In particular:
	\begin{align*}
	\widetilde{r}_m^2 \le \min \left\{ 2, 2 \widetilde{\beta}_{M-1} \left\| \widetilde{\us}_m \right\|_{\widetilde{\Vs}_{m-1}^{-1}}\right\} \le 2 \widetilde{\beta}_{M-1} \min\left\{1,  \left\| \widetilde{\us}_m \right\|_{\widetilde{\Vs}_{m-1}^{-1}} \right\}.
	\end{align*}
	Plugging this inequality into the second summation, we obtain:
	\begin{align*}
	    \sum_{m=1}^M \widetilde{r}_m^2 & \le 4 \widetilde{\beta}_{M-1}^2 \sum_{m=1}^M  \min\left\{ 1, \left\| \widetilde{\us}_m \right\|_{\widetilde{\Vs}_{m-1}^{-1}}^2 \right\}\\
	    & \le  8d  \widetilde{\beta}_{M-1}^2 \log \left( 1+ \frac{M U^2}{d\lambda} \right) \le 8d  {\beta}_{T-1}^2 \log \left( 1+ \frac{T U^2}{d\lambda} \right),
	\end{align*}
	where the last passage follows from the elliptic potential lemma~\citep[][Lemma 19.4]{lattimore2020bandit}. Putting all together, we obtain the inequality holding with probability at least $1-\delta$:
	\begin{align*}
	    R^{\text{off}}(\text{\algnameshort}, T) \le \sqrt{8d T  {\beta}_{T-1}^2 \left(1 + \frac{\log T}{\log \frac{1}{\overline{\rho}}} \right) \log \left( 1+ \frac{T U^2}{d\lambda} \right)},
	\end{align*}
	having observed that $\widetilde{\beta}_{M-1} \le \beta_{T-1}$
	We can also arrive at a problem-dependent regret bound, by setting $\Delta \coloneqq \inf_{\us \in \Us \inner{\hs}{\us} < \inner{\hs}{\us^*}} \,\inner{\hs}{\us^*-\us}$ (if it exists $>0$). Since the instantaneous regret is either $0$ or at least $\Delta$, we have:
	\begin{align*}
	    R^{\text{off}}(\text{\algnameshort},T) & \le \sum_{m=1}^M (H_m + 1) \frac{\widetilde{r}_m^2}{\Delta}  \\
	    & \le \frac{H_{M} + 1}{\Delta} 8d  \widetilde{\beta}_{M-1}^2 \log \left( 1+ \frac{M U^2}{d\lambda} \right) \\
	    & \le \frac{8d }{\Delta} \left(1 + \frac{\log T}{\log \frac{1}{\overline{\rho}}} \right) {\beta}_{T-1}^2 \log \left( 1+ \frac{T U^2}{d\lambda} \right).
	\end{align*}
	By setting $\delta = 1/T$, replacing the value of $\beta_{T-1}$, we obtain the offline regret in expectation, highlighting the dependence on $T$, $\overline{\rho}$, $d$, and $\sigma$ only:
	\begin{align*}
	    \E R^{\text{off}}(\text{\algnameshort},T) &\le \mathcal{O} \left( \frac{d \sigma \sqrt{T} (\log T)^{\frac{3}{2}}}{1-\overline{\rho}} + \frac{\sqrt{d T} (\log T)^2}{(1-\overline{\rho})^{\frac{3}{2}}} \right),
	\end{align*}
	where we used the fact that $\frac{1}{\log \frac{1}{\overline{\rho}}} \le \frac{1}{1-\overline{\rho}}$ and $\rho(\As) \le \overline{\rho}$.
\end{proof}

The following lemma relates the expected offline regret with the expected online regret.

\regretThr*

\begin{proof}
	The result is simply obtained by exploiting the offline regret bound of Theorem~\ref{thr:regretOff} and by upper bounding the expected regret using Lemma~\ref{lemma:offOnRel}.
\end{proof}

\section{Finite-Horizon Setting}\label{apx:finiteHorizon}
In this section, we compare the finite-horizon setting with the infinite-horizon one presented in the main paper. We shall show that under Assumption~\ref{ass:spectralNorm}, the two settings tend to coincide when the horizon is sufficiently large. Let us start by introducing the \emph{$H$--horizon expected average reward}, with $H \in \Nat$ being the optimization horizon:
\begin{align}\label{eq:eqDynFinite}
	J_H(\underline{\vpi}) \coloneqq  \E \left[\frac{1}{H} \sum_{t=1}^H y_t \right] \quad \text{where} \quad \begin{cases} \xs_{t+1} = \As \xs_t + \Bs \us_t + \epsilons_t\\
	y_t =  \inner{\vomega}{\xs_t} + \inner{\vtheta}{\us_t} + \eta_t\\
	\us_t = \vpi_t(H_{t-1})
	\end{cases}, \quad t \in [H],
\end{align}
where the expectation is taken \wrt the randomness of the state noise $\epsilons_t$ and reward noise $\eta_t$. We now show that the optimal policy for the finite-horizon setting is a non-stationary open-loop policy.

\begin{thr}[Optimal Policy for the $H$--Horizon Setting]
If $H \in \Nat$, an optimal policy $\underline{\vpi}^*_H = ({\vpi}^*_{H,t})_{t \in \dsb{H}}$ maximizing the $H$-horizon expected average reward $J(\underline{\vpi})$ as in Equation~\eqref{eq:eqDynFinite} is given by:
\begin{align*}
	\forall t \in \dsb{H},\quad \forall H_{t-1} \in \Hs_{t-1}: \qquad \vpi_{H,t}^*(H_{t-1}) = \us^*_{H,t} \quad \text{ where } \quad \us^*_{H,t} \in \argmax_{\us \in \Us}  \, \inner{\hs^{\dsb{0,H-t}}}{\us}.
\end{align*}
\end{thr}

\begin{proof}
We start by expressing for every $t \in \dsb{H}$ the reward $y_t$ as a function of the sequence of actions $\underline{\us} = (\us_1,\dots ,\us_{H})$ produced by a generic policy $\underline{\vpi}$. By exploiting Equation~\eqref{eq:markovDec} instanced with $H={t-1}$, we have:
\begin{align*}
	y_t = \sum_{s=0}^{t-1} \inner{\hs^{\{s\}}}{\us_{t-s}}  + \vomega^\transpose \As^{t-1} \xs_1 + \eta_t + \sum_{s=1}^{t-1} \vomega^\transpose \As^{s-1} \epsilons_{t-s}.
\end{align*}
By computing the expectation, using linearity, and recalling that the noises are zero-mean, we obtain:
\begin{align*}
	\E[y_t] = \sum_{s=0}^{t-1} \inner{\hs^{\{s\}}}{\E[\us_{t-s}]}   + \vomega^\transpose \As^{t-1} \E[\xs_1].
\end{align*}
By averaging over $t \in \dsb{H}$, we obtain the $H$-horizon expected average reward:
\begin{align}
	J_H(\underline{\vpi}) & = \frac{1}{H}\sum_{t=1}^H \E[y_t] \notag \\
	& = \frac{1}{H} \sum_{t=1}^H \sum_{s=0}^{t-1} \inner{\hs^{\{s\}}}{\E[\us_{t-s}]}  + \frac{1}{H} \sum_{t=1}^H \vomega^\transpose \As^{t-1} \E[\xs_1]  \notag 
	\\
	& = \frac{1}{H} \sum_{s=1}^H \left(\sum_{t=s}^H \hs^{\{t-s\}} \right)^\transpose \E[\us_{s}] +\frac{1}{H} \sum_{t=1}^H  \vomega^\transpose \As^{t-1} \E[\xs_1] \label{p:001}\\
	& = \frac{1}{H} \sum_{s=1}^H \inner{\hs^{\dsb{0,H-s}}}{ \E[\us_{s}]} +\frac{1}{H} \sum_{t=1}^H  \vomega^\transpose \As^{t-1} \E[\xs_1]. \label{p:002}
\end{align}
where line~\eqref{p:001} is obtained by renaming the indexes of the summations, and line~\eqref{p:002} comes from the definition of cumulative Markov parameter $\hs^{\dsb{0,H-s}}$. It is now simple to see, as no noise is present in the expression, that the performance $J_H(\underline{\vpi})$ is maximized by taking at each round $s \in \Nat$ an action $\us_s^* = \vpi_s^*(H_{s-1})$ such that whose expectation satisfies $\E[\us_{s}^*] =\argmax_{\E[\us_{s}]} \inner{\hs^{\dsb{0,H-s}}}{ \E[\us_{s}]}$. Clearly, we can take the deterministic action such that $\us_s^* = \E[\us_{s}^*]$.
\end{proof}

We now show that for sufficiently large $H$, the $H$-horizon expected average reward $J_H$ tends to coincide with the infinite-horizon  expected average reward.

\begin{prop}
Let $H \in \Nat$. Then, for every policy $\underline{\vpi}$ it holds that:
\begin{align*}
    \left| J_H(\underline{\vpi}) - J(\underline{\vpi}) \right| \le \frac{BU\Omega \Phi(\As)(1-\rho(\As)^H)}{H(1-\rho(\As))}.
\end{align*}
\end{prop}

\begin{proof}
    Consider two horizons $H < H' \in \Nat$, and let $(\us_1, \us_2, \dots)$ be the sequence of actions played by policy $\underline{\vpi}$. Using Equation~\eqref{p:002}, we have:
    \begin{align}
         J_H(\underline{\vpi}) - J_{H'}(\underline{\vpi}) & = \frac{1}{H} \sum_{s=1}^H \inner{\hs^{\dsb{0,H-s}}}{ \E[\us_{s}]} - \frac{1}{H'} \sum_{s=1}^{H'} \inner{\hs^{\dsb{0,H'-s}}}{ \E[\us_{s}]} \\
         & = \frac{1}{H} \sum_{s=1}^H \inner{\hs^{\dsb{0,H-s}} - \hs}{ \E[\us_{s}]} - \frac{1}{H'} \sum_{s=1}^{H'} \inner{\hs^{\dsb{0,H'-s}} - \hs}{ \E[\us_{s}]} \\
         & = - \frac{1}{H} \sum_{s=1}^H \inner{\hs^{\srb{H-s+1, +\infty}}}{ \E[\us_{s}]} + \frac{1}{H'} \sum_{s=1}^{H'} \inner{\hs^{\srb{H'-s+1,+\infty}} }{ \E[\us_{s}]}. 
    \end{align}
    As shown in Appendix~\ref{proof:0001}, we have that the second addendum vanishes as $H'$ approaches $+\infty$:
    \begin{align*}
         \frac{1}{H'} \left|\sum_{s=1}^{H'} \inner{\hs^{\srb{H'-s+1,+\infty}} }{ \E[\us_{s}]}\right| \rightarrow 0 \qquad \quad \text{when}\qquad \quad H' \rightarrow +\infty.
    \end{align*}
    Concerning the first addendum, we have:
    \begin{align*}
          \frac{1}{H} \left|\sum_{s=1}^H \inner{\hs^{\srb{H-s+1, +\infty}}}{ \E[\us_{s}]}\right| & \le \frac{U}{H} \sum_{s=1}^H \left\| \hs^{\srb{H-s+1, +\infty}}\right\|_2\\
          & \le \frac{BU\Omega \Phi(\As)}{H} \sum_{s=1}^H \rho(\As)^{H-s} \\
          & =  \frac{BU\Omega \Phi(\As)(1-\rho(\As)^H)}{H(1-\rho(\As))}.
    \end{align*}
\end{proof}

\section{System Identification}
\label{apx:systemidenfication}

This section presents a solution to identify matrices $\Abf$, $\Bbf$, $\Cbf$, and $\Dbf$ characterizing an LTI system starting from a single trajectory. We adopt a variant of the Ho-Kalman~\citep{ho1966effective} algorithm. We start from the identification method proposed by~\citet[][Section 3]{lale2020logarithmic}, where authors consider a system of the type (strictly proper):
\begin{align}
    \xs_{t+1} &= \Abf \xs_t + \Bbf \us_t + \epsilons_t, \label{eq:systemtilde_y}\\
    \widetilde{\ys}_{t} &= \Cbf \xs_t + \zs_t. \nonumber
\end{align}
Our setting can be seen as (not strictly proper):
\begin{align}
    \xs_{t+1} &= \Abf \xs_t + \Bbf \us_t + \epsilons_t, \label{eq:system_y}\\
    \ys_{t} &= \Cbf \xs_t + \Dbf \us_t + \zs_t, \nonumber
\end{align}
with $\xs_t, \epsilons_t \in \mathbb{R}^n$, $\us_t \in \mathbb{R}^p$, and $\ys_t, \zs_t \in \mathbb{R}^m$. The noise over state transition model $\epsilons_t$ and output $\zs_t$ are $\sigma^2$-subgaussian random variables.
We consider in this part the standard control problem notation adopted for LTI systems. The mapping to our problem is straightforward by considering $\Cbf = \vomega^\transpose$ and $\Dbf = \vtheta^\transpose$.
In predictive form, the system described in Equation~\eqref{eq:systemtilde_y} is:
\begin{align*}
    \widehat{\xs}_{t+1} &= \bar{\Abf} \widehat{\xs}_t + \Bbf \us_t + \Fbf \widetilde{\ys}_t, \\
    \widetilde{\ys}_{t} &= \Cbf \widehat{\xs}_t + \es_t,
\end{align*}
where:
\begin{align*}
    \bar{\Abf} &= \Abf - \Fbf \Cbf , \\
    \Fbf &= \Abf \Sigmabf \Cbf^\transpose (\Cbf \Sigmabf \Cbf^\transpose + \sigma^2 \Ibf )^{-1},
\end{align*}
and $\Sigmabf$ is the solution to the following DARE (Discrete Algebraic Riccati Equation):
\begin{equation*}
    \Sigmabf = \Abf \Sigmabf \Abf^\transpose - \Abf \Sigmabf \Cbf^\transpose (\Cbf \Sigmabf \Cbf^\transpose + \sigma^2 \Ibf )^{-1} \Cbf \Sigmabf \Abf^\transpose + \sigma^2 \Ibf .
\end{equation*}
In order to identify this LTI system, we want to detect a matrix $\widetilde{\mathcal{G}}_{y}$:
\begin{equation}
    \widetilde{\mathcal{G}}_{y} = \left[ \Cbf \Fbf \ \ \ \Cbf \bar \Abf \Fbf \ \ \ \dots \ \ \ \Cbf \bar \Abf^{H-1} \Fbf \ \ \ \Cbf \Bbf \ \ \ \Cbf \bar \Abf \Bbf \ \ \ \dots \Cbf \bar \Abf^{H-1} \Bbf \right].
\end{equation}
To identify through least squares method matrix $\widetilde{\mathcal{G}}_{y}$, we construct for each $t$, a vector $\widetilde{\phi}_t$:
\begin{equation}
    \widetilde{\phi}_t = \left[ \ys_{t-1}^\transpose \ \ \ \dots \ \ \ \ys_{t-H}^\transpose \ \ \ \us_{t-1}^\transpose \ \ \ \dots \ \ \ \us_{t-H}^\transpose \right]^\transpose \in \mathbb{R}^{(m+p)H}.
\end{equation}
The system output $\widetilde{\ys}_t$ can be rewritten as:
\begin{equation*}
    \widetilde{\ys}_t = \widetilde{\mathcal{G}}_{y} \widetilde{\phi}_t + \es_t + \Cbf \Abf^H \xs_{t-H}.
\end{equation*}

The output of the system under analysis (Equation~\ref{eq:system_y}) is:
\begin{align*}
\ys(t) = \widetilde{\ys}_{t} + \Dbf \us_t = \widetilde{\mathcal{G}}_{y} \widetilde{\phi}_t + \Dbf \us_t + \es_t + \Cbf \Abf^H \xs_{t-H}
\end{align*}
We can incorporate the contribution of $\Dbf \us_t$ in $\widetilde{\mathcal{G}}_y$ obtaining $\mathcal{G}_y$:
\begin{equation*}
    \mathcal{G}_{y} = \left[ \Cbf \Fbf \ \ \ \Cbf \bar \Abf \Fbf \ \ \ \dots \ \ \ \Cbf \bar \Abf^{H-1} \Fbf \ \ \ \Dbf \ \ \ \Cbf \Bbf \ \ \ \Cbf \bar \Abf \Bbf \ \ \ \dots \Cbf \bar \Abf^{H-1} \Bbf \right].
\end{equation*}
The related vector $\phi_t$ is:
\begin{equation}
    \phi_t = \left[ \ys_{t-1}^\transpose \ \ \ \dots \ \ \ \ys_{t-H}^\transpose \ \ \ \us_t^\transpose \ \ \ \us_{t-1}^\transpose \ \ \ \dots \ \ \ \us_{t-H}^\transpose \right]^\transpose \in \mathbb{R}^{(m+p)H + p}.
\end{equation}

The best value of $\mathcal{G}_{y}$ can be found through regularized least squares as in~\citet[][Equation 10]{lale2020logarithmic}:
\begin{equation}
    \widehat{\mathcal{G}}_y = \argmin_{X} \lambda \| \mathbf{X} \|^2_F + \sum_{\tau = t - H}^{t} \| \ys_\tau - \mathbf{X} \phi_\tau \|^2_2,
\end{equation}
where $\| \cdot \|_F$ represents the Frobenius norm. 

The matrix $\Dbf$ can be directly retrieved from $\widehat{\mathcal{G}}_y$. In order to get matrices $\Abf$, $\Bbf$, and $\Cbf$, we remove the values related to $\Dbf$ from $\widehat{\mathcal{G}}_y$ and we retrieve $\widetilde{\mathcal{G}}_y$. From now on, we refer to the algorithm proposed in~\citet[][Appendix B]{lale2020logarithmic}.

\section{Integration on Numerical Simulations}
\label{apx:exper}

This section is divided in three parts. First, in Section~\ref{apx:notes_on_baselines}, we provide additional information about the baselines, their hyperparameters and the optimistic bounds. Second, in Section~\ref{apx:notes_on_realworldexp}, we provide all the matrices and vectors generalized to run the real-world experiment. Third, in Section~\ref{apx:additional_exp}, we provide further results for the simulations presented in Section~\ref{subsec:exp_synt}.

\subsection{Additional Notes on the Baselines}
\label{apx:notes_on_baselines}
As mentioned in Section~\ref{sec:numericalsimulations}, the chosen baselines are \linucb~\citep{abbasi2011improved}, \dlinucb~\citep{RussacVC19}, \artwo~\citep{chen2021dynamic}, \expthree~\citep{auer1995gambling}, \batchexpthree~\citep{dekel2012online,auer1995gambling} and the \manualexpert (the latter available only in the case of real-world data). All the hyperparameters, whenever possible, are set as prescribed in the original papers.
The bounds used for the exploration are adjusted in order to be able to fairly compete in this setting, and are considered as follows:
\begin{align*}
	& \beta_t^\texttt{Lin-UCB} \coloneqq \overline{c}_2 \sqrt{\lambda} + \sqrt{2 \overline{\sigma}^2 \left( \log \left(\frac{1}{\delta} \right) + \frac{d}{2} \log \left(1 + \frac{t U^2}{d \lambda} \right) \right)}, \\
	& \beta_t^\texttt{D-Lin-UCB} \coloneqq \overline{c}_2 \sqrt{\lambda} + \sqrt{2 \overline{\sigma}^2 \left( \log \left(\frac{1}{\delta} \right) + \frac{d}{2} \log \left(1 + \frac{t U^2}{d \lambda} \left( \frac{1- \gamma^{2t}}{1 - \gamma^2} \right) \right) \right)},
\end{align*}
where $\overline{c}_2$ and $\overline{\sigma}^2$ are as prescribed in Section~\ref{sec:regret}, and the hyperparameter $\gamma$ of \dlinucb is tuned.

For \artwo, the hyperparameter $\alpha$, describing the correlation over time is considered equal to $\rho(\As)$.

In the case of \expthree, the rewards are rescaled in order to make them range in $[0,1]$ with high probability, as follows: 
\begin{align*}
	\overline{r}_t = \frac{r_t + 2 \xi}{4 \xi}, \qquad \text{ where } \qquad \xi = \left( \Theta + \frac{ \Omega B}{1 - \rho (\As)} \right) U.
\end{align*}
Furthermore, in the case of \batchexpthree, the batch dimension $k$ is considered as:
\begin{align*}
	k = \left\lceil \frac{\log{M}}{\log{1/\overline{\rho}}} \right\rceil ,
\end{align*}
where $M$ is the one defined in Algorithm~\ref{alg:alg}~(line~\ref{line:define_M}). This batch size $k$ ensures that, at each time $t$, the contribution of actions $\us_s$ is negligible, with $s \in \dsb{t-k-1}$. The rewards collected in the same batch are averaged and transformed as in \expthree.

\subsection{Further Information on the Real-world Setting}
\label{apx:notes_on_realworldexp}
The real-world setting is generalized through a dataset containing real data related to the budgets invested in each advertising platform (\ie the $\us_t$) and the overall generated conversions (\ie the $y_t$) collected from three of the most important advertising platforms of the web (\texttt{Facebook}, \texttt{Google}, and \texttt{Bing}), related to a large number of campaigns for a value of more than $5$ Million USD over $2$ years. Starting from such data, we generalized the best model by means of a specifically designed variant of the Ho-Kalman algorithm~\citep{ho1966effective}, as described in Appendix~\ref{apx:systemidenfication}. We used the matrices estimated with Ho-Kalman to build up a simulator. The resulting system has $\rho (\As) = 0.67$, and is characterized as follows:
\begin{align*}
    \As = 
	\begin{pmatrix}
		0.38 & 0.33 & 0.6 \\
        0.07 & 0.76 & -0.54 \\
        0.18 & 0.34 & 0.05 
    \end{pmatrix} ,
    \qquad
    \Bs = 
	\begin{pmatrix}
            0.14 & 0.34 & -0.05 \\
            -0.17 & 0.03 & -0.01 \\
            0.04 & -0.09 & 0.17
    \end{pmatrix} ,
	\qquad
    \boldsymbol{\omega} = 
	\begin{pmatrix}
        -0.61 \\
		-0.04 \\
		-0.13 
    \end{pmatrix} ,
    \qquad
    \boldsymbol{\theta} = 
	\begin{pmatrix}
        0.13 \\
		0.41 \\
		0.02
    \end{pmatrix} .
\end{align*}

\subsection{Additional Numerical Simulations}
\label{apx:additional_exp}
These additional results are obtained in the setting presented in Section~\ref{subsec:exp_synt}. However, here, we want to analyze the behavior of \algnameshort and the other bandit baselines at different magnitudes of noise in both the state transition model and the output. The noise in this simulation is a zero-mean Gaussian noise with $\sigma \in \{ 0.001, 0.01, 0.1 \}$.

\textbf{Results}~~Figure~\ref{fig:regret_wrt_sigma} shows the results of the experiment for the different values of $\sigma$. It is clearly visible how \algnameshort performs in almost the same way no matter the noise to which the system is subject, always leading to sub-linear regret. 
On the other hand, the cumulative regret of both \linucb and \dlinucb is different in every simulation we perform. Indeed, with a low level of noise (Figure~\ref{fig:regret_suppl_sigma0001}) reaches linear regret and does not converge, while for large values of noise, it converges very quickly (Figure~\ref{fig:regret_suppl_sigma01}). This is due to the nature of the confidence bound of linear bandits, which is not able to take into account such a complex scenario and leads to no guarantees in this setting. \expthree, \batchexpthree, and \artwo are not able to reach the optimum in this scenario, independently from the noise magnitude  $\sigma$, and provide large values of (linear) regret.

\begin{figure}[t!]
    \centering
    \subfloat[$\sigma=0.001$]{\resizebox{0.32\linewidth}{!}{\includegraphics[]{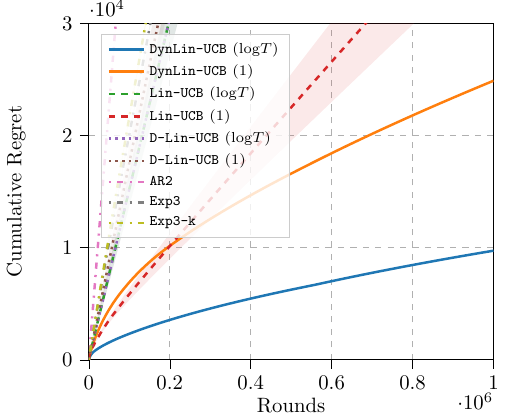}} \label{fig:regret_suppl_sigma0001}}
    \hfill
    \subfloat[$\sigma=0.01$]{\resizebox{0.32\linewidth}{!}{\includegraphics[]{content/img/final_sigma001.pdf}} \label{fig:regret_suppl_sigma001}}
    \hfill
    \subfloat[$\sigma=0.1$]{\resizebox{0.32\linewidth}{!}{\includegraphics[]{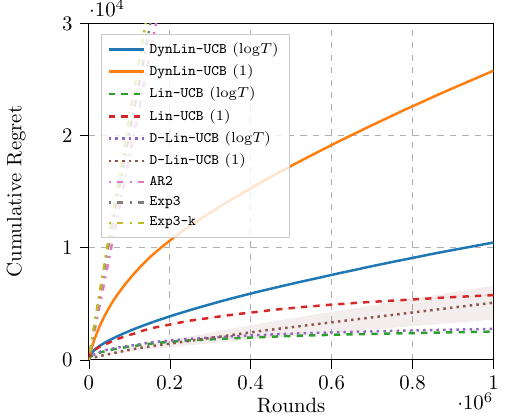}} \label{fig:regret_suppl_sigma01}}
    \caption{Performance of \algnameshort, \linucb, \dlinucb, \artwo, \expthree and \batchexpthree at different values of $\sigma$. (50 runs, mean $\pm$ std)}
    \label{fig:regret_wrt_sigma}
\end{figure}

\subsection{Computational Time}
\label{apx:computational_time}
The code used for the results provided in this section has been run on an Intel(R) I5 8259U @ 2.30GHz CPU with $8$ GB of LPDDR3 system memory. The operating system was macOS $12.2.1$, and the experiments have been run on \textit{Python} $3.9.7$. A single run of \algnameshort takes $110$ seconds to run.
It is worth noting that the time complexity of \algnameshort is upper-bounded by the one of \linucb.

\end{document}